\documentclass[journal]{IEEEtran}
\usepackage{amsmath,graphicx}
\usepackage{algorithmicx}
\usepackage{algpseudocode}

\usepackage{amsfonts}
\usepackage{amssymb,amsmath}
\usepackage{amsthm}
\usepackage{cite}
\usepackage{cleveref}

\usepackage{algorithm}
\usepackage{footnote}
\usepackage{float}
\usepackage[caption=false]{subfig}
\usepackage{caption}
\usepackage{multicol}
\usepackage{blindtext}
\usepackage{color}

\newtheorem{theorem}{Theorem}[section]
\newtheorem{corollary}{Corollary}[theorem]
\newtheorem{lemma}[theorem]{Lemma}
\newtheorem{definition}{Definition}[theorem]
\newtheorem{assumption}{Assumption}
\newtheorem{proposition}{Proposition}
\newenvironment{hproof}{%
	\proof}{\endproof}

\ifCLASSINFOpdf
\else
\fi
\begin{document}
%
\title{K-medoids Clustering of Data Sequences with Composite Distributions}
%
%
%

\author{Tiexing~Wang,
		Qunwei Li,~\IEEEmembership{Student Member,~IEEE},
		Donald J. Bucci,
		Yingbin Liang,~\IEEEmembership{Senior Member,~IEEE},\\
		Biao Chen,~\IEEEmembership{Fellow,~IEEE},
		Pramod K. Varshney,~\IEEEmembership{Life Fellow,~IEEE}
\thanks{Some parts of the preliminary versions of this manuscript were published in Proc.\ ICASSP 2018 and Proc.\ CISS 2018.}
\thanks{T. Wang, Q. Li, B. Chen and P. K. Varshney are with the department
of Electrical Engineering and Computer Science, Syracuse University, Syracuse, NY, 13244, USA (e-mail: \{twang17, qli33, bichen, varshney\}@syr.edu).}
\thanks{D. J. Bucci is with Lockheed Martin - Advanced Technology Labs, Cherry Hill, NJ, 08002, USA (e-mail: Donald.J.Bucci.Jr@lmco.com).}
\thanks{Y. Liang is with the department of Electrical and Computer Engineering, The Ohio State University, Columbus, OH 43210, USA (e-mail: liang.889@osu.edu).}
\thanks{This material is based upon work supported in part by the Defense Advanced Research Projects Agency under Contract No. HR0011-16-C-0135 and by the dynamic data driven Applications Systems (DDDAS) program of AFOSR under grant number FA9550-16-1-0077. The work of Y. Liang was also supported in part by the National Science Foundation under grant CCF-1801855. The work of B. Chen was also supported in part by the National Science Foundation under grant CNS-1731237.}
}

\maketitle

\begin{abstract}
This paper studies clustering of data sequences using the k-medoids algorithm. All the data sequences are assumed to be generated from \emph{unknown} continuous distributions, which form clusters with each cluster containing a composite set of closely located distributions (based on a certain distance metric between distributions). The maximum intra-cluster distance is assumed to be smaller than the minimum inter-cluster distance, and both values are assumed to be known. The goal is to group the data sequences together if their underlying generative distributions (which are unknown) belong to one cluster. Distribution distance metrics based k-medoids algorithms are proposed for known and unknown number of distribution clusters. Upper bounds on the error probability and convergence results in the large sample regime are also provided. It is shown that the error probability decays exponentially fast as the number of samples in each data sequence goes to infinity. The error exponent has a simple form regardless of the distance metric applied when certain conditions are satisfied. In particular, the error exponent is characterized when either the Kolmogrov-Smirnov distance or the maximum mean discrepancy are used as the distance metric. Simulation results are provided to validate the analysis.
\end{abstract}



%
\IEEEpeerreviewmaketitle

\section{Introduction}
%
%
%
%
\IEEEPARstart{T}{HIS} paper aims to cluster sequences generated by \emph{unknown} continuous distributions into classes so that each class contains all the sequences generated from the same (composite) distribution cluster. By sequence, we mean a set of feature observations generated by an underlying probability distribution. Here each distribution cluster contains a set of distributions
that are close to each other, whereas different clusters are assumed to be far away from each other based on a certain distance metric between distributions. To be more concrete, we assume that the maximum intra-cluster distance (or its upper bound) is smaller than the minimum inter-cluster distance (or its lower bound). This assumption is necessary for the clustering problem to be meaningful. The distance metrics used to characterize the difference of data sequences and corresponding distributions are assumed to have certain properties summarized in Assumption \ref{assump} in Section \ref{sec:sys_model}.

Such unsupervised learning problems have been widely studied  \cite{Bishop06,Barber12}. The problem is typically solved by applying clustering methods, e.g., k-means clustering \cite{Lloyd1982,Stei2000,Kanu2002} and k-medoids clustering\cite{kaufman1987,Laan2003,park2009}, where the data sequences are viewed as multivariate data with Euclidean distance as the distance metric. These clustering algorithms usually require the knowledge of the number of clusters and they differ in how the initial centers are determined. One reasonable way is to choose a data sequence as a center if it has the largest minimum distances to all the existing centers \cite{Katsavounidis1994,Wang:ICASSP2018,Wang:CISS2018}. Alternatively, all the initial centers can be randomly chosen \cite{MORENOSAEZ2014}. With the number of clusters unknown, there are typically two alternative approaches for clustering. One starts with a small number of clusters, e.g., $1$, which is an underestimate of the true number, and proceed to split the existing clusters until convergence \cite{LOPEZ2015,Wang:CISS2018}. The authors in \cite{LOPEZ2015} assumed a maximum number of clusters and the threshold for clustering depended on a pre-determined significance level of the two sample Kolmogorov-Smirnov (KS) test whereas the algorithm proposed in \cite{Wang:CISS2018} did not assume a maximum number of clusters and the threshold for clustering was a function of the intra-cluster and inter-cluster distances. Alternatively, one may start with an overestimated the number of clusters, e.g., every sequence is treated as a cluster, and proceed to merge clusters that are deemed close to each other \cite{Wang:CISS2018}. The algorithms in \cite{Katsavounidis1994,MORENOSAEZ2014,LOPEZ2015} were all validated by simulation results without carrying out an analysis of the error probability.

There are some key differences between the k-means algorithm and the k-medoids algorithm. The k-means algorithm minimizes a sum of squared Euclidean distances. Meanwhile, the k-medoids algorithm assigns data sequences as centers and minimizes a sum of arbitrary distances, which makes it more robust to outliers and noise \cite{Ng2002,kaufman2009}. Moreover, the k-means algorithm requires updating the distances between data sequences and the corresponding centroids in every iteration whereas the k-medoids algorithm only requires the pairwise distances of the data sequences, which can be computed before hand. Thus, the k-medoids algorithm outperforms the k-means algorithm in terms of computational complexity as the number of sequences increases \cite{velmurugan2010}. 

Most prior research focused on computational complexity analysis, whereas the error probability and the performance comparison of different clustering algorithms were typically studied through simulations, e.g., \cite{Laan2003,Sheng2006,park2009,velmurugan2010}. This paper attempts to theoretically analyze the error probability for the k-medoids algorithm especially in the asymptotic region. Furthermore, in contrast to previous studies, which frequently used vector norms as the distance metric, e.g., Euclidean distance, our study adopts the distance metrics between distributions for clustering in order to capture the statistical models of data sequences considered in this paper. This formulation based on a distributional distance metric is uniquely suited to the proposed clustering problem, where each data point, i.e., each data sequence, represents a probability distribution and each cluster is a collection of closely related distributions, i.e., composite hypotheses.

Various distance metrics that take the distribution properties into account can be reasonable choices, e.g., KS distance \cite{Wang:CISS2018,Wang:ICASSP2018,LOPEZ2015,MORENOSAEZ2014} and maximum mean discrepancy (MMD) \cite{Gangeh2014}. 
Our previous work \cite{Wang:CISS2018,Wang:ICASSP2018} has shown that the KS distance based k-medoids algorithms are exponentially consistent for both known and unknown number of clusters. In this paper, we consider a much more general framework and instead of focusing on a specific distance metric such as KS distance in our prior work \cite{Wang:CISS2018,Wang:ICASSP2018}, we consider any distance metric that satisfies Assumption \ref{assump} for clustering. The rationale is the following: with any distance metric that captures the statistical model of data sequences, one is likely to observe that for a large sample size, 1) sequences generated from the same distribution cluster are close to each other, and 2) sequences generated from different distribution clusters are far away from each other. Thus, if the minimum distance of different distribution clusters is greater than the maximum diameter of all the distribution clusters defined in Section \ref{sec:sys_model}, then it becomes unlikely to make a clustering error as the sample size increases. In this paper, we develop k-mediods distribution clustering algorithms where the distances between distributions are selected to capture the underlying statistical model of the data. We analyze the error probability of the proposed algorithms, which takes the form of exponential decay as the sample size increases under Assumption \ref{assump}.
Furthermore, beyond the KS distance, we also consider another distance metric namely MMD and show that they both satisfy Assumption \ref{assump} so that the error probability of the proposed algorithms decays exponentially fast if the KS distance and MMD are used as the distance metrics. 
\par

We note that recent studies \cite{Li2013,Bu2017,Liang16tsp_a,Liang16tsp_b,mai2012} of anomaly detection problems and classification also took into account the statistical model of the data sequences. Our focus here is on the clustering problem, leading to error performance analysis that is substantially different from that in \cite{Li2013,Bu2017,Liang16tsp_a,Liang16tsp_b,mai2012}.

The rest of the paper is organized as follows. In Section \ref{sec:sys_model}, the system model of the clustering problem, the preliminaries of the KS distance and MMD, and notations are introduced. The clustering algorithm given the number of clusters and the corresponding upper bound on the error probability are provided in Section \ref{sec:knownK}, followed by the results of the clustering algorithms with an unknown number of clusters in Section \ref{sec:unknownK}. The simulation results for the KS and MMD based algorithms are provided in Section \ref{sec:numerical_result}.

\section{system model and preliminaries}\label{sec:sys_model}
\subsection{Clustering Problem}
Suppose there are $K$ distribution clusters denoted by $\mathcal{P}_{k}$ for $k=1,\ldots,K$, where $K$ is fixed. Define the intra-cluster distance of $\mathcal{P}_{k}$ and the inter-cluster distance between $\mathcal{P}_{k}$ and $\mathcal{P}_{k^{\prime}}$ for $k\neq k^{\prime}$ respectively as
\begin{equation}\label{eq:cluster_setup1}
\begin{aligned}
d\big(\mathcal{P}_{k}\big)& =\sup_{p_{i},p_{i^{\prime}}\in\mathcal{P}_{k}}d\big(p_{i},p_{i^{\prime}}\big),\\
d\big(\mathcal{P}_{k},\mathcal{P}_{k^{\prime}}\big)& =\inf_{p_{i}\in \mathcal{P}_{k},p_{i^{\prime}}\in\mathcal{P}_{k^{\prime}}}d\big(p_{i},p_{i^{\prime}}\big),
\end{aligned}
\end{equation}
where $d\left(\cdot,\cdot\right)$ is a distance metric between distributions, e.g., the KS distance or MMD defined later in \eqref{eq:defKS} and \eqref{eq:mmdpq} respectively. Then $d\left(\mathcal{P}_{k}\right)$ and $d\left(\mathcal{P}_{k},\mathcal{P}_{k^{\prime}}\right)$ represent the diameter of $\mathcal{P}_{k}$ and the distance between $\mathcal{P}_{k}$ and $\mathcal{P}_{k^{\prime}}$, respectively. Define
\begin{equation}\label{eq:cluster_setup2}
\begin{aligned}
d_{L} &= \max_{k=1,\ldots,K}d\big(\mathcal{P}_{k}\big),\\
d_{H} &= \min_{k \neq k^{\prime}}d\big(\mathcal{P}_{k},\mathcal{P}_{k^{\prime}}\big),\\ 
\Sigma & = d_{H}+d_{L},\\
\Delta & = d_{H} -d_{L}.
\end{aligned}
\end{equation}
Table \ref{table:notation} summarizes the notations of the generalized form of distances defined in \eqref{eq:cluster_setup1} and \eqref{eq:cluster_setup2} which will be used in the following discussion.
\begin{table}
	\caption{Notations}
	\centering
	\begin{tabular}{ | c | c | c |}
		\hline
		general & KS & MMD \\ \hline
		$d\left(\mathcal{P}_{k}\right)$ & $d_{KS}\left(\mathcal{P}_{k}\right)$ & $\mathbb{MMD}^{2}\left(\mathcal{P}_{k}\right)$\\ \hline
		$d\left(\mathcal{P}_{k},\mathcal{P}_{k^{\prime}}\right)$ & $d_{KS}\left(\mathcal{P}_{k},\mathcal{P}_{k^{\prime}}\right)$ & $\mathbb{MMD}^{2}\left(\mathcal{P}_{k},\mathcal{P}_{k^{\prime}}\right)$\\ \hline
		$d_{L}$ & $d_{L,ks}$ &  $d_{L,mmd}$ \\ \hline
		$d_{H}$ & $d_{H,ks}$ &  $d_{H,mmd}$ \\ \hline
		$\Delta$ &  $\Delta_{ks}$  & $\Delta_{mmd}$ \\ \hline
		$\Sigma$ &  $\Sigma_{ks}$  & $\Sigma_{mmd}$ \\ \hline
	\end{tabular}	
	\label{table:notation}
\end{table}
\par
Suppose $M_{k}$ data sequences are generated from the distributions in $\mathcal{P}_{k}$, and hence a total of $\sum_{k=1}^{K}M_{k} = M$ sequences are to be clustered. Without loss of generality, assume that each sequence $\mathbf{x}_{k,j_{k}}$ consists of $n$ independently identically distributed (i.i.d.) samples generated from $p_{k,j_{k}}\in\mathcal{P}_{k}$ for $k=1,\ldots,K$ and $j_{k}\in\{1,\ldots, M_{k}\}$. The $i$-th observation in $\mathbf{x}_{k,j_{k}}$ is denoted by $\mathbf{x}_{k,j_{k}}[i]\in \mathbb{R}^{m}$, where $m<\infty$ and $i\in\{1,\ldots,n\}$. Note that for any fixed $k$, $p_{k,j_{k}}$'s are not necessarily distinct. Namely, for a fixed $k$, some $\mathbf{x}_{k,j_{k}}$'s can be generated from the same distribution. Although the following discussion assumes that all the data sequences have the same length, our analysis can be easily extended to the case with different sequence lengths by replacing $n$ with the minimum sequence length. In order to analyze the error probability of the clustering algorithm, we introduce an assumption relating to the concentration property of the distance metrics:
\begin{assumption}\label{assump}
For any distribution clusters $\{\mathcal{P}_{1},\ldots,\mathcal{P}_{K}\}$ and any sequences $\mathbf{x}_{k,j_{k}}\sim p_{k,j_{k}}$, $\mathbf{x}_{k,j_{k}^{\prime}}\sim p_{k,j_{k}^{\prime}}$ and $\mathbf{x}_{k^{\prime},j_{k^{\prime}}}\sim p_{k^{\prime},j_{k^{\prime}}}$, where $k\neq k^{\prime}$, the following inequalities hold:
\begin{subequations}\label{eq:assump1}
\begin{align}
&d_{L} < d_{H},\label{eq:KSassumptionHT}\\
&P\big(d(\mathbf{x}_{k,j_{k}},\mathbf{x}_{k^{\prime},j_{k^{\prime}}})\leq d_{th}\big) \leq a_{1}e^{-bn}\:\forall d_{th}\in \left(d_{L},d_{H}\right), \label{eq:dist_condition2}\\
&P\big(d(\mathbf{x}_{k,j_{k}},\mathbf{x}_{k,j_{k}^{\prime}})>d_{th}\big) \leq a_{2}e^{-bn}\:\:\forall d_{th}\in 
\left(d_{L},d_{H}\right), \label{eq:dist_condition1}\\
&P\big(d(\mathbf{x}_{k,j_{k}},\mathbf{x}_{k,j_{k}^{\prime}})\geq d(\mathbf{x}_{k,j_{k}},\mathbf{x}_{k^{\prime},j_{k^{\prime}}})\big) \leq a_{3}e^{-bn}, \label{eq:dist_condition3}
\end{align}
\end{subequations}
where $a_{i}$'s are some constants independent of distributions, $b$ ($>0$) is a function of $d_{th}$ and $n$ is the sample size.
\hfill\qedsymbol
\end{assumption}
Here \eqref{eq:KSassumptionHT} guarantees that the lower bound of inter-cluster distances is greater than the upper bound of intra-cluster distances. \eqref{eq:dist_condition2} guarantees that the probability that the distance between two sequences generated from different distribution clusters is smaller than $d_{H}$ decays exponentially fast. \eqref{eq:dist_condition1} guarantees that the probability that the distance between two sequences generated from the same distribution cluster is greater than $d_{L}$ decays exponentially fast. \eqref{eq:dist_condition3} guarantees that given two sequences generated from the same cluster and a third sequence generated from another distribution cluster, the probability that the first sequence is actually closer to the third sequence decays exponentially fast.
\par
A clustering output is said to be incorrect if and only if the sequences generated by different distribution clusters are assigned to the same cluster, or
sequences generated by the same distribution cluster are assigned to more than one cluster. Denote by $P_{e}$ the error probability of a clustering algorithm. A clustering algorithm is said to be {\em consistent} if
\begin{equation*}
\lim_{n\rightarrow\infty}P_{e}=0,
\end{equation*}
where $n$ is the sample size. The algorithm is said to be {\em exponentially consistent} if
\begin{equation*}
B = \lim_{n\rightarrow\infty}-\frac{1}{n}\log P_{e}>0.
\end{equation*}
For the case where a clustering algorithm is exponentially consistent, we are also interested in characterizing the error exponent $B$.\par
\subsection{Preliminaries of KS distance}
Suppose $\mathbf{x}$ is generated by the distribution $p$, where $\mathbf{x}[i]\in \mathbb{R}$. Then the empirical cumulative distribution function (c.d.f.) induced by $\mathbf{x}$ is given by
\begin{equation}\label{eq:def-cdf}
F_{\mathbf{x}}\left(a\right)=\frac{1}{n}\sum_{i=1}^{n}1_{[-\infty,a]}\left(\mathbf{x}[i]\right),
\end{equation}
where $1_{[-\infty,x]}\left(\cdot\right)$ is the indicator function. Let the c.d.f. of $p$ evaluated at $a$ be $F_{p}\left(a\right)$. The KS distance is defined as
\begin{equation}\label{eq:defKS}
d_{KS}\left(p,q\right) = \sup_{a\in \mathbb{R}}|F_{p}\left(a\right)- F_{q}\left(a\right)|,
\end{equation}
where $F_{p}\left(a\right)$ and $F_{q}\left(a\right)$ can be either c.d.f's of distributions or empirical c.d.f.'s induced by sequences.
\begin{proposition}\label{proposition1}
	The KS distance satisfies \eqref{eq:dist_condition2} - \eqref{eq:dist_condition3} if $d_{L,ks}<d_{H,ks}$.
\end{proposition}
\begin{proof}
	See Lemmas \ref{lemma:KStest3}, \ref{lemma:KStest2}, and \ref{lemma:KStest4} in Appendix \ref{app:lemmas}.
\end{proof}
\subsection{Preliminaries of MMD}
Suppose $\mathcal{P}$ includes a class of probability distributions, and suppose $\mathcal{H}$ is the reproducing kernel Hilbert space (RKHS) associated with a kernel $g\left(\cdot,\cdot\right)$. Define a mapping from $\mathcal{P}$ to $\mathcal{H}$ such that each distribution $p\in \mathcal{P}$ is mapped into an element in $\mathcal{H}$ as follows
\[\mu_p\left(\cdot\right)=\mathbb{E}_p [g\left(\cdot,x\right)]=\int g\left(\cdot,x\right)dp\left(x\right), \]
where $\mu_p\left(\cdot\right)$ is referred to as the {\em mean embedding} of the distribution $p$ into the Hilbert space $\mathcal{H}$. Due to the reproducing property of $\mathcal{H}$, it is clear that $\mathbb{E}_p[f]=\langle \mu_p,f \rangle_{\mathcal{H}}$ for all $f \in \mathcal{H}$.\par
It has been shown in \cite{Fuku2008,Srip2008,Fuku2009,Srip2010} that for many RKHSs such as those associated with Gaussian and Laplace kernels, the mean embedding is injective, i.e., each $p \in \mathcal{P}$ is mapped to a unique element $\mu_p \in \mathcal{H}$. In this way, many machine learning problems with unknown distributions can be solved by studying mean embeddings of probability distributions without actually estimating the distributions, e.g., \cite{Song2013,Smola2007,Liang16tsp_a,Liang16tsp_b}.  In order to distinguish between two distributions $p$ and $q$, the authors in \cite{Gretton2012} introduced the following notion of MMD based on the mean embeddings $\mu_p$ and $\mu_q$ of $p$ and $q$, respectively:
\begin{equation}\label{eq:mmdpq}
\mathbb{MMD}\left(p,q\right):=\|\mu_p-\mu_q\|_{\mathcal{H}}.
\end{equation}
An unbiased estimator of $\mathbb{MMD}^2\left(p,q\right)$ based on $\mathbf{x}$ and $\mathbf{y}$ which consist of $n$ and $m$ samples, respectively, is given by
\begin{equation}\label{eq:mmdu}
\begin{aligned}
&\quad\:\mathbb{MMD}^2\left(\mathbf{x},\mathbf{y}\right)\\
&=\frac{1}{n\left(n-1\right)}\sum_{i=1}^n\sum_{j\neq i}^n g\left(\mathbf{x}[i],\mathbf{x}[j]\right) +\frac{1}{m\left(m-1\right)}\\
&\qquad \sum_{i=1}^m\sum_{j\neq i}^m g\left(\mathbf{y}[i],\mathbf{y}[j]\right)-\frac{2}{nm}\sum_{i=1}^n\sum_{j=1}^m g\left(\mathbf{x}[i],\mathbf{y}[j]\right).
\end{aligned}
\end{equation}
Assume that the kernel is bounded, i.e., $0 \le g\left(x,y\right) \le \mathbb K$, where $\mathbb{K}$ is finite.
\begin{proposition}\label{proposition2}
	The MMD distance satisfies \eqref{eq:dist_condition2} - \eqref{eq:dist_condition3} if $d_{L,mmd}<d_{H,mmd}$.
\end{proposition}
\begin{proof}
	See Lemmas \ref{lemma:MMD2}, \ref{lemma:MMD1} and \ref{lemma:MMD3} in Appendix \ref{app:lemmas}.
\end{proof}
\subsection{Additional Notations}
The following notations are used in the algorithms and the corresponding proofs. Let $\mathcal{C}_{l}^{t}$ be the $l$-th cluster obtained at the $t$-th cluster update step and let $\mathbf{c}_{l}^{t,a}$, $\mathbf{c}_{l}^{t,e}$ and $\mathbf{c}_{l}^{t,s}$ be the centers of the $l$-th cluster obtained by the center update step, merge step and split step of the $t$-th iteration respectively for $t\geq 1$. Moreover, let $\mathcal{C}_{l}^{0}$ be the $l$-th cluster obtained at the initialization step and $\mathbf{c}_{l}^{0,a}$ be the corresponding center. Let $\hat{K}^{t}$ ($t\geq 1$) be the number of centers before the $t$-th cluster update step and the $t$-th split step. Moreover, use $\hat{K}^{0}$ to denote the number of centers obtained at the center initialization step. For simplicity, all the superscripts are omitted in the following discussion when there is no ambiguity.\par
To further simplify the notation in the algorithms and the proofs, let $\{\mathbf{y}_{i}\}_{i=1}^{M}$ denote the data sequence set $\{\mathbf{x}_{k,j_{k}}\}_{k=1,j_{k}=1}^{K,M_{k}}$. However, the one-to-one mapping from $\{\mathbf{y}_{i}\}_{i=1}^{M}$ onto $\{\mathbf{x}_{k,j_{k}}\}_{k=1,j_{k}=1}^{K,M_{k}}$ is not fixed, i.e., given a fixed $i$, $\mathbf{y}_{i}$ can be any sequence in $\{\mathbf{x}_{k,j_{k}}\}_{k=1,j_{k}=1}^{K,M_{k}}$ unless other constraints are imposed. Denote by $\mathbf{y}_{i}\sim\mathcal{P}_{k}$ if $\mathbf{y}_{i}$ is generated from a distribution $p\in\mathcal{P}_{k}$.
Furthermore, define a set of integers
\begin{equation*}
\begin{aligned}
I_{k_{1}}^{k_{2}} &= \{k_{1},\ldots,k_{2}\},
\end{aligned}
\end{equation*}
where $k_{1},\:k_{2}\in \mathbb{Z}^{+}$ and $k_{1}<k_{2}$.\par

\section{Known number of clusters}\label{sec:knownK}

In this section, we study the clustering algorithm for known $K$, the number of clusters. The method proposed in \cite{Katsavounidis1994} is used for center initialization, as described in Algorithm \ref{KS-Initial-known-C}. The initial $K$ centers are chosen sequentially such that the center of the $k$-th cluster is the sequence that has the largest minimum distance to the previous $k-1$ centers. The clustering algorithm itself is presented in Algorithm \ref{K-means-known-C}. Given the centers, each sequence is assigned to the cluster for which the sequence has the minimum distance to the center. For a given cluster, a sequence is assigned as the center subsequently if the sum of its distances to all the sequences in the cluster is the smallest. The algorithm continues until the clustering result converges. \par
\begin{algorithm}[!bt]
	\caption{Initialization with known $K$}
	\label{KS-Initial-known-C}
	\begin{algorithmic}[1]
		\State \textbf{Input}: Data sequences $\{\mathbf{y}_{i}\}_{i=1}^{M}$, number of clusters $K$.
		\State \textbf{Output}: Partitions $\{\mathcal{C}_{k}\}_{k=1}^{K}$.
		\State \{Center initialization\}
		\State Arbitrarily choose one $\mathbf{y}_{i}$ as $\mathbf{c}_{1}$.
		\For {$k=2\text{ to }K$}
		\State $\mathbf{c}_{k}\leftarrow \text{arg}\max_{\mathbf{y}_{i}}\left(\min_{l\in I_{1}^{k-1}} d\left(\mathbf{y}_{i},\mathbf{c}_{l}\right)\right)$
		\EndFor
		\State \{Cluster initialization\}
		\State Set $\mathcal{C}_{k}\leftarrow \emptyset$ for $1\leq k\leq K$.
		\For {$j=1\text{ to }M$}
		\State $\mathcal{C}_{l}\leftarrow \mathcal{C}_{l}\cup\{\mathbf{y}_{i}\}$, where
		$l=\text{arg}\min_{l\in I_{1}^{K}} d\left(\mathbf{y}_{i},\mathbf{c}_{l}\right)$
		\EndFor
		\State Return $\{\mathcal{C}_{k}\}_{k=1}^{K}$
	\end{algorithmic}
\end{algorithm}

\begin{algorithm}[!bt]
	\caption{Clustering with known $K$}
	\label{K-means-known-C}
	\begin{algorithmic}[1]
		\State \textbf{Input}: Data sequences $\{\mathbf{y}_{i}\}_{i=1}^{M}$, number of clusters $K$.
		\State \textbf{Output}: Partition set $\{\mathcal{C}_{k}\}_{k=1}^{K}$.
		\State Initialize $\{\mathcal{C}_{k}\}_{k=1}^{K}$ by Algorithm \ref{KS-Initial-known-C}.
		\While{not converge}
		\State \{Center update\}
		\For {$k=1\text{ to }K$}
		\State $\mathbf{c}_{k}\leftarrow \arg\min_{\mathbf{y}_{i}\in \mathcal{C}_{k}}\sum_{\mathbf{y}_{j^{\prime}}\in \mathcal{C}_{k}}d\left(\mathbf{y}_{i},\mathbf{y}_{j^{\prime}}\right)$
		\EndFor
		\State $\{\text{Cluster update}\}$
		\For {$j=1\text{ to }M$}
		\If {$\mathbf{y}_{i}\in \mathcal{C}_{k^{\prime}}$ and $d\left(\mathbf{y}_{i},\mathbf{c}_{k}\right)< d\left(\mathbf{y}_{i},\mathbf{c}_{k^{\prime}}\right)$}
		\State $\mathcal{C}_{k}\leftarrow \mathcal{C}_{k}\cup\{\mathbf{y}_{i}\}$ and $\mathcal{C}_{k^{\prime}}\leftarrow \mathcal{C}_{k^{\prime}}\setminus\{\mathbf{y}_{i}\}$.
		\EndIf
		\EndFor
		\EndWhile
		\State Return $\{\mathcal{C}_{k}\}_{k=1}^{K}$
	\end{algorithmic}
\end{algorithm}

The following theorem provides the convergence guarantee for Algorithm \ref{K-means-known-C} via an upper bound on the error probability.
\begin{theorem}\label{theorem:KStest1}
	Algorithm \ref{K-means-known-C} converges after a finite number of iterations. Moreover, under Assumption \ref{assump},
	the error probability of Algorithm \ref{K-means-known-C} after $T$ iterations is upper bounded as follows
	\begin{equation*}
	P_{e}\leq M^{2}\left(a_{1} + a_{2} + \left(T+1\right)a_{3}\right)e^{-bn},
	\end{equation*}
	where $a_{1}$, $a_{2}$, $a_{3}$ and $b$ are as defined in Assumption \ref{assump}.
\end{theorem}
\begin{hproof}
	The idea of proving the upper bound on the error probability is as follows. We first prove that the error probability at the initialization step decays exponentially. Note that the event that an error occurs during the first $T$ iterations is the union of the event that an error occurs at the $t$-th step and the previous $t-1$ iterations are correct for $t=1,\ldots,T$. Thus, if we prove that the error probability at the $t$-th step \emph{given} correct updates from the previous iterations decays exponentially, then so does the error probability of the algorithm by the union bound argument. See Appendix \ref{proof:theorem:KStest1} for details.
\end{hproof}
Theorem \ref{theorem:KStest1} shows that for any given $K$, any distance metric satisfying Assumption \ref{assump} yields an exponentially consistent k-medoids clustering algorithm with the error exponent $b$.
\begin{corollary}\label{coro:knownK}
Suppose the KS distance and MMD are used for Algorithms \ref{KS-Initial-known-C} and \ref{K-means-known-C}, then
	\begin{equation*}
	\begin{aligned}
	&P_{e}^{KS} \leq M^{2}\left(6T+14\right)\exp\left(-\frac{n\Delta_{ks}^{2}}{8}\right),\\
	&P_{e}^{MMD}\leq M^{2}\left(T+3\right)\exp \left(-\frac{n\Delta_{mmd}^2}{256\mathbb{K}^2}\right).
	\end{aligned}
	\end{equation*}
\end{corollary}
\begin{proof}
	By Propositions \ref{proposition1} and \ref{proposition2}, the upper bound on the error probability of Algorithm \ref{K-means-known-C} applies to the KS distance and MMD. Thus, the corollary is obtained by substituting the values specified in Lemmas \ref{lemma:KStest2} - \ref{lemma:MMD3} in the upper bound.
\end{proof}
Corollary \ref{coro:knownK} implies that Algorithm \ref{K-means-known-C} is exponentially consistent under KS and MMD distance metrics with an error exponent no smaller than $\frac{\Delta_{ks}^{2}}{8}$ and $\frac{\Delta_{mmd}^2}{256\mathbb{K}^2}$, respectively.

\section{Unknown number of clusters}\label{sec:unknownK}
In this section, we propose the merge- and split-based algorithms for estimating the number of clusters as well as grouping the sequences.
\subsection{Merge Step}
If a distance metric satisfies \eqref{eq:dist_condition1} and two sequences generated by distributions within the same cluster are assigned as centers, then, with high probability, the distance between the two centers is small. This is the premise of the clustering algorithm based on merging centers that are close to each other.\par
The proposed approach is summarized in Algorithms \ref{KS-Initial-unknown-UC} and \ref{K-means-unknown-UC}. There are two major differences between Algorithms \ref{KS-Initial-unknown-UC} and \ref{K-means-unknown-UC} and Algorithms \ref{KS-Initial-known-C} and \ref{K-means-known-C}. First, the center initialization step of Algorithm \ref{KS-Initial-unknown-UC} keeps generating an increasing number of centers until all the sequences are close to one of the existing centers. Second, an additional Merge Step in Algorithm \ref{K-means-unknown-UC} helps to combine clusters if the corresponding centers have small distances between each other.\par
\begin{algorithm}[!bt]
	\caption{Merge-based initialization with unknown $K$}
	\label{KS-Initial-unknown-UC}
	\begin{algorithmic}[1]
		\State \textbf{Input}: Data sequences $\{\mathbf{y}_{i}\}_{i=1}^{M}$ and threshold $d_{th}$.
		\State \textbf{Output}: Partitions $\{\mathcal{C}_{k}\}_{k=1}^{\hat{K}}$.
		\State $\{\text{Center initialization}\}$
		\State Arbitrarily choose one $\mathbf{y}_{i}$ as $\mathbf{c}_{1}$ and set $\hat{K}=1$.
		\While {$\min_{i\in I_{1}^{M}}\left(\min_{k\in I_{1}^{\hat{K}}} d\left(\mathbf{y}_{i},\mathbf{c}_{k}\right)\right)>d_{th}$}
		\State $\mathbf{c}_{\hat{K}+1}\leftarrow\text{arg}\max_{\mathbf{y}_{i}}\left(\min_{k\in I_{1}^{\hat{K}}} d\left(\mathbf{y}_{i},\mathbf{c}_{k}\right)\right)$
		\State $\hat{K}\leftarrow \hat{K}+1$
		\EndWhile
		\State Clustering initialization specified in Algorithm \ref{KS-Initial-known-C}.
		\State Return $\{\mathcal{C}_{k}\}_{k=1}^{\hat{K}}$
	\end{algorithmic}
\end{algorithm}

\begin{algorithm}[!bt]
	\caption{Merge-based clustering with unknown $K$}
	\label{K-means-unknown-UC}
	\begin{algorithmic}[1]
		\State \textbf{Input}: Data sequences $\{\mathbf{y}_{i}\}_{i=1}^{M}$ and threshold $d_{th}$.
		\State \textbf{Output}: Partition set $\{\mathcal{C}_{k}\}_{k=1}^{\hat{K}}$.
		\State Initialize $\{\mathcal{C}_{k}\}_{k=1}^{\hat{K}}$ by Algorithm \ref{KS-Initial-unknown-UC}.
		\While{not converge}
		\State Center update specified in Algorithm \ref{K-means-known-C}.
		\State $\{\text{Merge Step}\}$
		\For {$k_{1}, k_{2}\in \{1,\ldots,\hat{K}\}$ and $k_{1}\neq k_{2}$}
		\If {$d\left(\mathbf{c}_{k_{1}},\mathbf{c}_{k_{2}}\right)\leq d_{th}$}
		\If {$\sum_{\mathbf{y}_{i}\in \mathcal{C}_{k_{1}}}d\left(\mathbf{c}_{k_{2}},\mathbf{y}_{i}\right) < \sum_{\mathbf{y}_{i}\in \mathcal{C}_{k_{2}}}d\left(\mathbf{c}_{k_{1}},\mathbf{y}_{i}\right)$}
		\State $\mathcal{C}_{k_{2}} \leftarrow \mathcal{C}_{k_{1}}\cup\mathcal{C}_{k_{2}}$ and delete $\mathbf{c}_{k_{1}}$ and $\mathcal{C}_{k_{1}}$.
		\Else
		\State $\mathcal{C}_{k_{1}} \leftarrow \mathcal{C}_{k_{1}}\cup\mathcal{C}_{k_{2}}$ and delete $\mathbf{c}_{k_{2}}$ and $\mathcal{C}_{k_{2}}$.
		\EndIf
		\State $\hat{K}\leftarrow \hat{K} - 1$.
		\EndIf
		\EndFor
		\State Cluster update specified in Algorithm \ref{K-means-known-C}.
		\EndWhile
		\State Return $\{\mathcal{C}_{k}\}_{k=1}^{\hat{K}}$
	\end{algorithmic}
\end{algorithm}
\begin{theorem}\label{theorem:KStest2}
	Algorithm \ref{K-means-unknown-UC} converges after a finite number of iterations. Moreover, under Assumption \ref{assump}, the error probability of Algorithm \ref{K-means-unknown-UC} after $T$ iterations is upper bounded as follows
	\begin{equation*}
	P_{e}\leq M^{2}\left(\left(T+1\right)a_{1} +a_{2} + \left(T+1\right)a_{3}\right)e^{-bn},
	\end{equation*}
	where $a_{1}$, $a_{2}$, $a_{3}$ and $b$ are as defined in Assumption \ref{assump}.
\end{theorem}
\begin{proof}
	The proof shares the same idea as that of Theorem \ref{theorem:KStest1}. See Appendix \ref{proof:theorem:KStest2} for details.
\end{proof}
Theorem \ref{theorem:KStest2} shows that the merge-based algorithm is exponentially consistent under any distance metric satisfying Assumption \ref{assump} with the error exponent $b$.
\begin{corollary}\label{coro:unknownK-merge}
Suppose the KS distance and MMD are used with $d_{th} = \frac{\Sigma_{ks}}{2}$ and $d_{th}=\frac{\Sigma_{mmd}}{2}$. Then the error probability of Algorithm \ref{K-means-unknown-UC} after $T$ iterations is upper bounded as follows
\begin{equation*}
\begin{aligned}
&P_{e}^{KS}\leq M^{2}\left(10T+14\right) \exp{\left(-\frac{n\Delta_{ks}^{2}}{8}\right)},\\
&P_{e}^{MMD}\leq M^{2}\left(2T+3\right) \exp{\left(-\frac{n\Delta_{mmd}^{2}}{256\mathbb{K}^{2}}\right)}.
\end{aligned}
\end{equation*}
\end{corollary}
\begin{proof}
By Propositions \ref{proposition1} and \ref{proposition2}, the upper bound on the error probability of Algorithm \ref{K-means-unknown-UC} in \ref{theorem:KStest2} applies to the KS distance and MMD. Thus, the corollary is obtained by substituting the values specified in Lemmas \ref{lemma:KStest2} - \ref{lemma:MMD3} in the upper bound.
\end{proof}
Corollary \ref{coro:unknownK-merge} implies that Algorithm \ref{K-means-unknown-UC} is exponentially consistent under KS and MMD distance metrics with an error exponent no smaller than $\frac{\Delta_{ks}^{2}}{8}$ and $\frac{\Delta_{mmd}^2}{256\mathbb{K}^2}$, respectively.

\subsection{Split Step}
Suppose a cluster contains sequences generated by different distributions and the center is generated from $p\in\mathcal{P}_{k}$. Then if the distance metric satisfies \eqref{eq:dist_condition2}, the probability that the distances between sequences generated from distribution clusters other than $\mathcal{P}_{k}$ and the center is small decays as the sample size increases. Therefore, it is reasonable to begin with one cluster and then split a cluster if there exists a sequence in the cluster that has a large distance to the center. The corresponding algorithm is summarized in Algorithm \ref{K-means-unknown-spl}.\par
\begin{algorithm}[!bt]
	\caption{Split-based clustering with unknown $K$}
	\label{K-means-unknown-spl}
	\begin{algorithmic}[1]
		\State \textbf{Input}: Data sequences $\{\mathbf{y}_{i}\}_{i=1}^{M}$ and threshold $d_{th}$.
		\State \textbf{Output}: Partition set $\{\mathcal{C}_{k}\}_{k=1}^{\hat{K}}$.
		\State  $\mathcal{C}_{1}=\{\mathbf{y}_{i}\}_{i=1}^{M}$, $\hat{K}=1$ and find $\mathbf{c}_{1}$ by center update specified in Algorithm \ref{K-means-known-C}.
		\While{not converge}
		\State {\{Split Step\}}
		\If {$\max_{k\in I_{1}^{\hat{K}},\: \mathbf{y}_{i}\in \mathcal{C}_{k}}d\left(\mathbf{c}_{k},\mathbf{y}_{i}\right)>d_{th}$}
		\State $\hat{K}\leftarrow \hat{K}+1$.
		\State $k=\arg\max_{k\in I_{1}^{\hat{K}}}\left(\max_{ \mathbf{y}_{i}\in \mathcal{C}_{k}}d\left(\mathbf{c}_{k},\mathbf{y}_{i}\right)\right)$
		\State $\mathbf{c}_{\hat{K}}\leftarrow\arg\max_{\mathbf{y}_{i}\in \mathcal{C}_{k}}d\left(\mathbf{c}_{k},\mathbf{y}_{i}\right)$
		\EndIf
		\State Cluster update specified in Algorithm \ref{K-means-known-C}.
		\EndWhile
		\State Return $\{\mathcal{C}_{k}\}_{k=1}^{\hat{K}}$
	\end{algorithmic}
\end{algorithm}
\begin{definition}\label{def:correct_split}
	Suppose Algorithm \ref{K-means-unknown-spl} obtains $\hat{K}$ clusters at the $t$-th iteration, where $\hat{K}<K$ and $\hat{K}=t$ or $t+1$. Then the correct clustering update result is that each cluster contains all the sequences generated from the distribution cluster that generates the center.
\end{definition}
\begin{theorem}\label{theorem:KStest3}
	Algorithm \ref{K-means-unknown-spl} converges after a finite number of iterations. Moreover, under Assumption \ref{assump}, the error probability of Algorithm \ref{K-means-unknown-spl} after $T$ iterations is upper bounded as follows
	\begin{equation*}
	P_{e}\leq M^{2}T\left(a_{1} + a_{2} + a_{3}\right)e^{-bn},
	\end{equation*}
	where $a_{1}$, $a_{2}$, $a_{3}$ and $b$ are as defined in Assumption \ref{assump}.
\end{theorem}
\begin{hproof}
	An error occurs at the $t$-th iteration if and only if the $\hat{K}$-th center is generated from distribution clusters that generated the previous centers or the clustering result is incorrect. Note that the error event of the first $T$ iterations is the union of the events that an error occurs at the $t$-th iteration while the clustering results in the previous $t-1$ iterations are correct for $t=1,\ldots,T$. Similar to the proof of Theorem \ref{theorem:KStest1}, the error probability is bounded by the union bound. See Appendix \ref{proof:theorem:KStest3} for more details.
\end{hproof}
Theorem \ref{theorem:KStest3} shows that the split-based algorithm is exponentially consistent under any distance metric satisfying Assumption \ref{assump} with the error exponent $b$.
\begin{corollary}\label{coro:unknownK-split}
		Suppose the KS distance and MMD are used with $d_{th} = \frac{\Sigma_{ks}}{2}$ and $d_{th}=\frac{\Sigma_{mmd}}{2}$. Then the error probability of Algorithm \ref{K-means-unknown-spl} after $T$ iterations is upper bounded as follows
		\begin{equation*}
		\begin{aligned}
		&P_{e}^{KS}\leq 14M^{2}T\exp{\left(-\frac{n\Delta_{ks}^{2}}{8}\right)},\\
		&P_{e}^{MMD} \leq 3M^{2}T\exp{\left(-\frac{n\Delta_{mmd}^{2}}{256\mathbb{K}^{2}}\right)}.
		\end{aligned}
		\end{equation*}
\end{corollary}
\begin{proof}
	By Propositions \ref{proposition1} and \ref{proposition2}, the upper bound on the error probability of Algorithm \ref{K-means-unknown-spl} in Theorem \ref{theorem:KStest3} applies to the KS distance and MMD. Thus, the corollary is obtained by substituting the values specified in Lemmas \ref{lemma:KStest2} - \ref{lemma:MMD3} in the upper bound.
\end{proof}
Corollary \ref{coro:unknownK-split} implies that Algorithm \ref{K-means-unknown-spl} is exponentially consistent under KS and MMD with an error exponent no smaller than $\frac{\Delta_{ks}^{2}}{8}$ and $\frac{\Delta_{mmd}^2}{256\mathbb{K}^2}$, respectively.
\section{Numerical Results}\label{sec:numerical_result}
In this section, we provide some simulation results given $K=5$ and $M_{k}=3$ for $k=1,\ldots,5$. Moreover, $\mathbf{x}_{k,j_{k}}[i]\in \mathbb{R}$. Gaussian distributions $\mathcal{N}\left(\mu_{k,j_{k}},\sigma^{2}\right)$ and Gamma distributions $\Gamma\left(a_{k,j_{k}},b\right)$ are used in the simulations. The probability density function (p.d.f.) of $\Gamma(a,b)$ is defined as
\begin{equation*}
\begin{aligned}
f\left(x;\alpha,\beta\right) & = \frac{1}{\beta^{\alpha}\Gamma\left(\alpha\right)}x^{\alpha-1}\exp{\left(-\frac{x}{\beta}\right)}\quad \left(x>0\right),
\end{aligned}
\end{equation*}
where $\alpha>0$, $\beta>0$ and $\Gamma\left(\cdot\right)$ is the Gamma function, respectively. Specifically, the parameters of the distributions are $\mu_{k,j_{k}}\in\{k-\delta,k,k+\delta\}$, $\sigma^{2} = 1$, $\alpha_{k,j_{k}}\in\{2.5k+1-\delta,2.5k+1,2.5k+1+\delta\}$ and $\beta=1$, where $\delta =0$ and $0.1$. Note that when $\delta=0$, all the sequences generated from the same distribution cluster are generated from a single distribution. The exponential kernel function is used in the simulations for the MMD distance, i.e.,
\begin{equation}\label{eq:simu_kernel}
g\left(x,y\right) = e^{-\frac{|x-y|}{2}}.
\end{equation}

\subsection{Known Number of Clusters}
Simulation results for a known number of clusters are shown in Fig.~\ref{fig:algorithm2}. One can observe from the figures that by using both the KS distance and MMD, $\log P_{e}$ is a linear function of the sample size, i.e., $P_{e}$ is exponentially consistent. Moreover, the logarithmic slope of $P_{e}$ with respect to $n$, i.e., the quantity $-\frac{\log P_{e}}{n}$, increases as $\delta$ becomes smaller, which, in the current simulation setting, implies a larger $\Delta$. \par
Furthermore, a good distance metric for Algorithm \ref{K-means-known-C} depends on the underlying distributions. This is because the underlying distributions have different distances under the KS distance and MMD, which results in different values of $\Delta_{ks}$ and $\Delta_{mmd}$.

	\begin{figure}[t]
		\centering
		\subfloat[Gaussian distributions\label{subfig:1-1}]{%
			\includegraphics[width=0.5\textwidth, height=0.15\textheight]{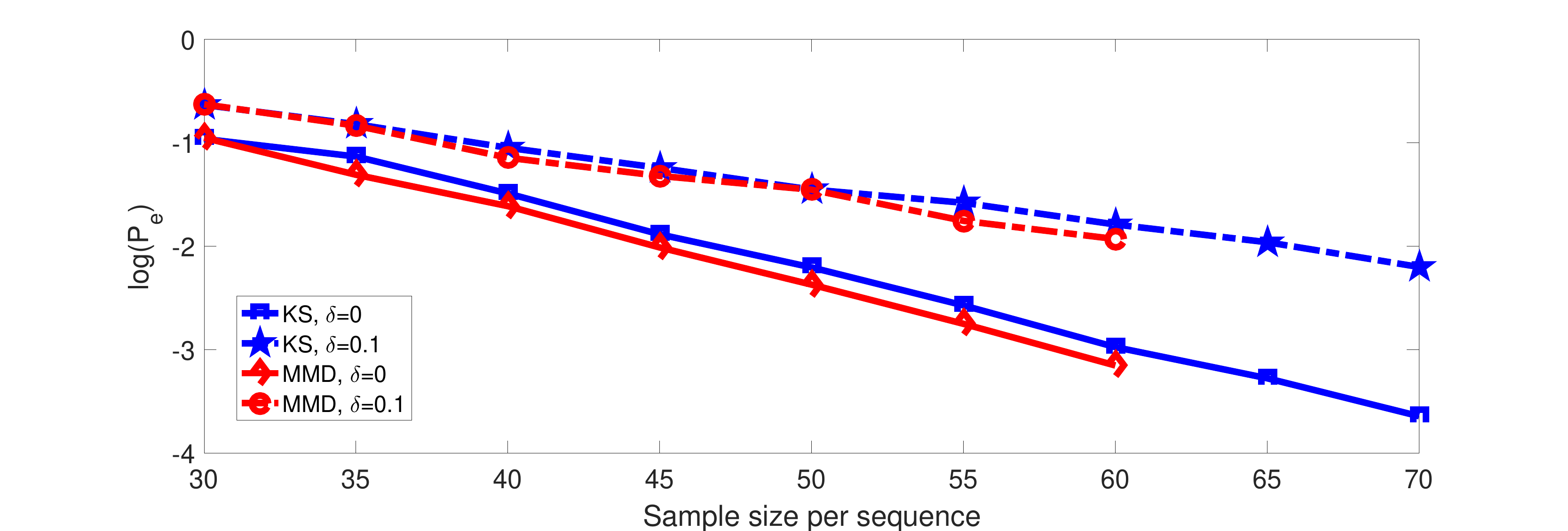}
		}
		\hfill
		\subfloat[Gamma distributions \label{subfig:1-2}]{%
			\includegraphics[width=0.5\textwidth, height=0.15\textheight]{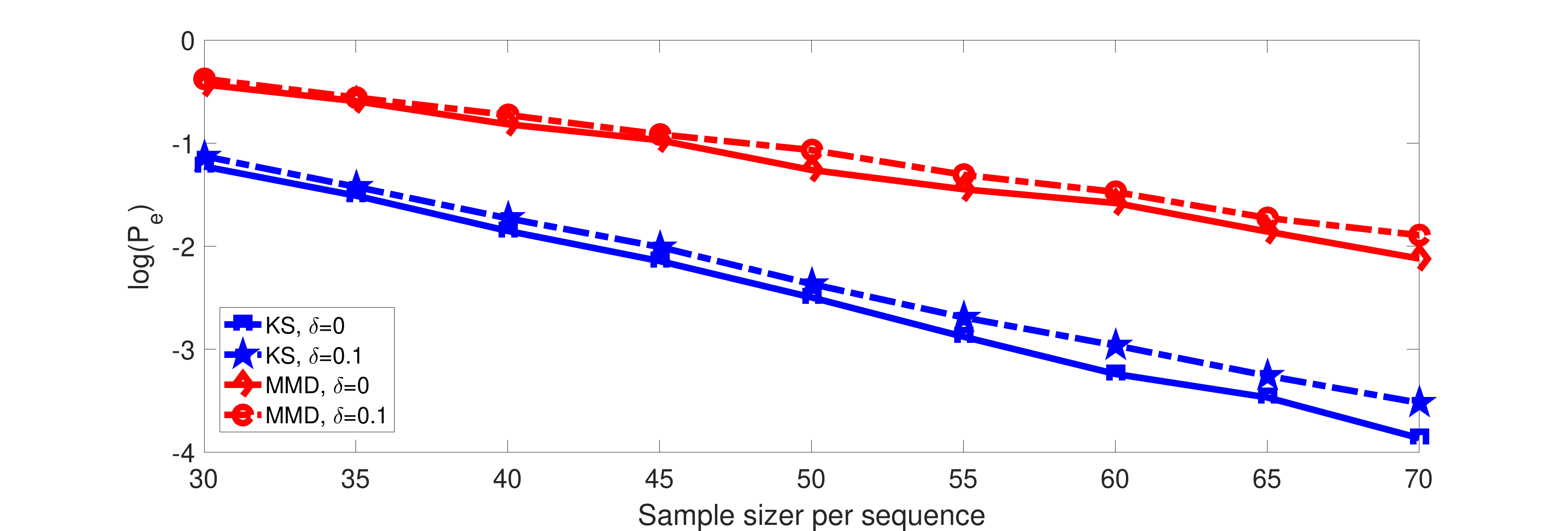}
		}
		\caption{Performance of Algorithm  2}
		\label{fig:algorithm2}
	\end{figure}


\subsection{Unknown Number of Clusters}
With an unknown number of distribution clusters, the threshold $d_{th}$ specified in Corollaries \ref{coro:unknownK-merge} and \ref{coro:unknownK-split} are used in the simulation. The performance of Algorithms \ref{K-means-unknown-UC} and \ref{K-means-unknown-spl} for the KS distance and MMD are shown in Figs.~\ref{fig:algorithm4_5_KS} and \ref{fig:algorithm4_5_MMD}, respectively. Given the KS distance and MMD, $\log {P_{e}}$'s are linear functions of the sample size when the sample size is large and larger $\Delta$ implies a larger slope of $\log {P_{e}}$. Furthermore, given the same value of $\delta$, Algorithms \ref{K-means-unknown-UC} and \ref{K-means-unknown-spl} have similar performance under the KS distance whereas Algorithm \ref{K-means-unknown-spl} outperforms Algorithm \ref{K-means-unknown-UC} under MMD given $\delta=0$ and $0.1$. Intuitively, smaller $\delta$ implies larger $\Delta$ in the current simulation setting, thereby should result in better clustering performance for a given sample size. However, Fig.~\ref{fig:algorithm4_5_KS}\subref{fig:algorithm4_5_KS-b} indicates that Algorithms \ref{K-means-unknown-UC} and \ref{K-means-unknown-spl} with KS distance performs better with $\delta=0.1$ than that with $\delta=0$ when the sample size is small. This is likely due to the fact that the KS distance between the two sequences is always lower bounded by $\frac{1}{n}$. Thus, with small sample sizes, Algorithms \ref{K-means-unknown-UC} and \ref{K-means-unknown-spl} are likely to overestimate the number of clusters. This can be mitigated by the increased threshold $d_{th}$ to control merging/splitting of cluster centers. 

\begin{figure}[!htb]
	\subfloat[Gaussian distributions\label{subfig:2-1}]{%
		\includegraphics[width=0.5\textwidth, height=0.15\textheight]{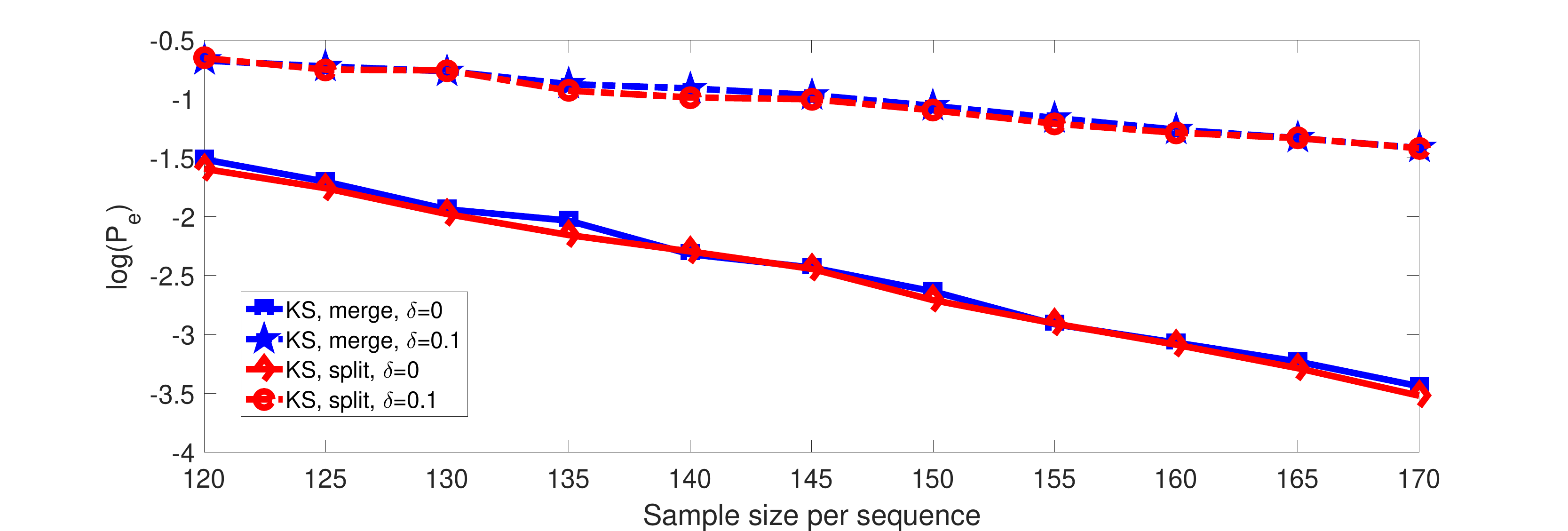}
	}
	\hfill
	\subfloat[Gamma distributions \label{subfig:2-2}]{%
		\includegraphics[width=0.5\textwidth, height=0.15\textheight]{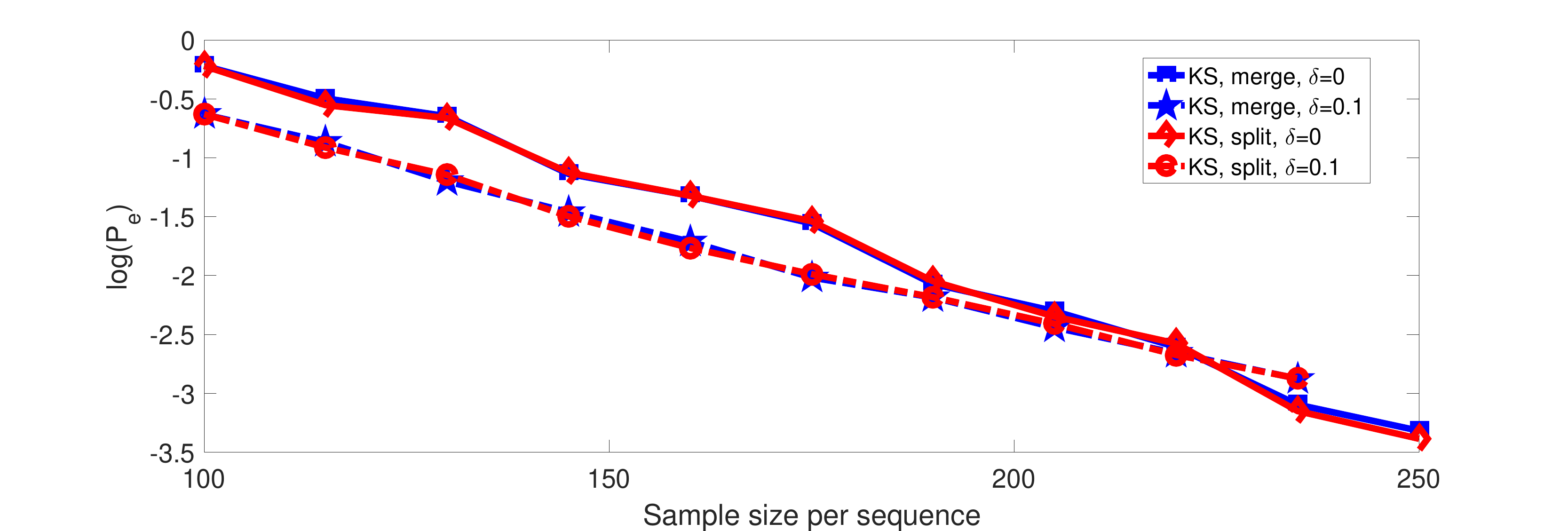}
		\label{fig:algorithm4_5_KS-b}
	}
	\caption{Performance of Algorithms \ref{K-means-unknown-UC} and \ref{K-means-unknown-spl} for the KS distance}
	\label{fig:algorithm4_5_KS}
\end{figure}

\begin{figure}[!htb]
	\subfloat[Gaussian distributions\label{subfig:3-1}]{%
		\includegraphics[width=0.5\textwidth, height=0.15\textheight]{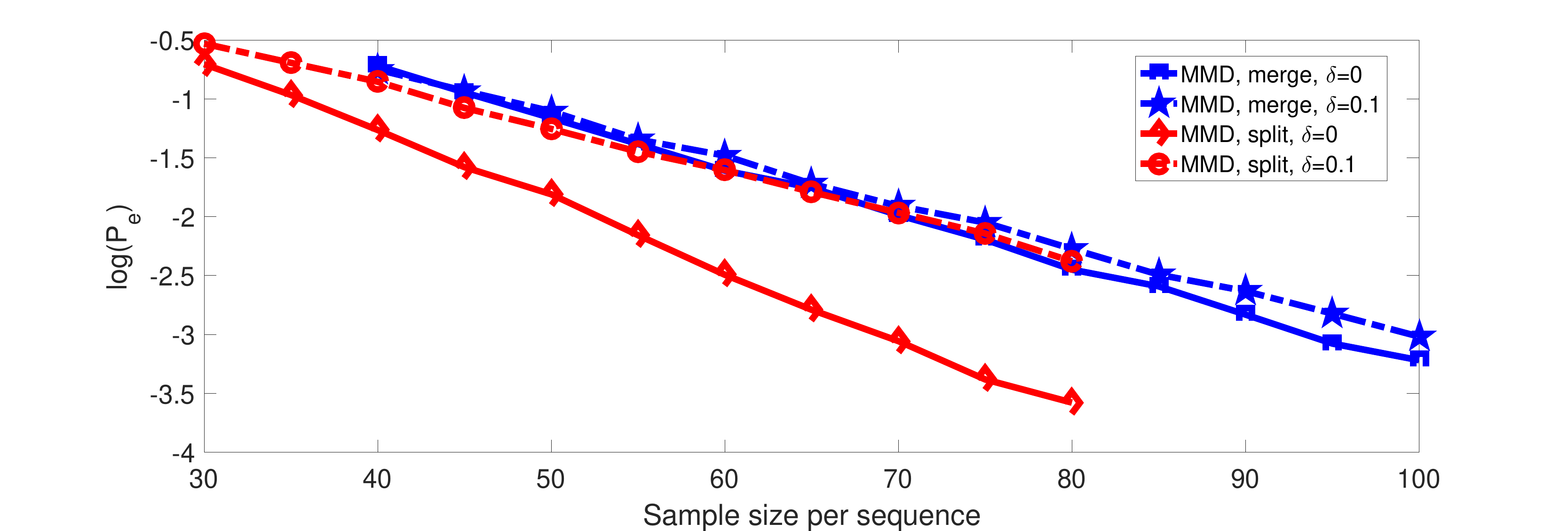}
	}
	\hfill
	\subfloat[Gamma distributions\label{subfig:3-2}]{%
		\includegraphics[width=0.5\textwidth, height=0.15\textheight]{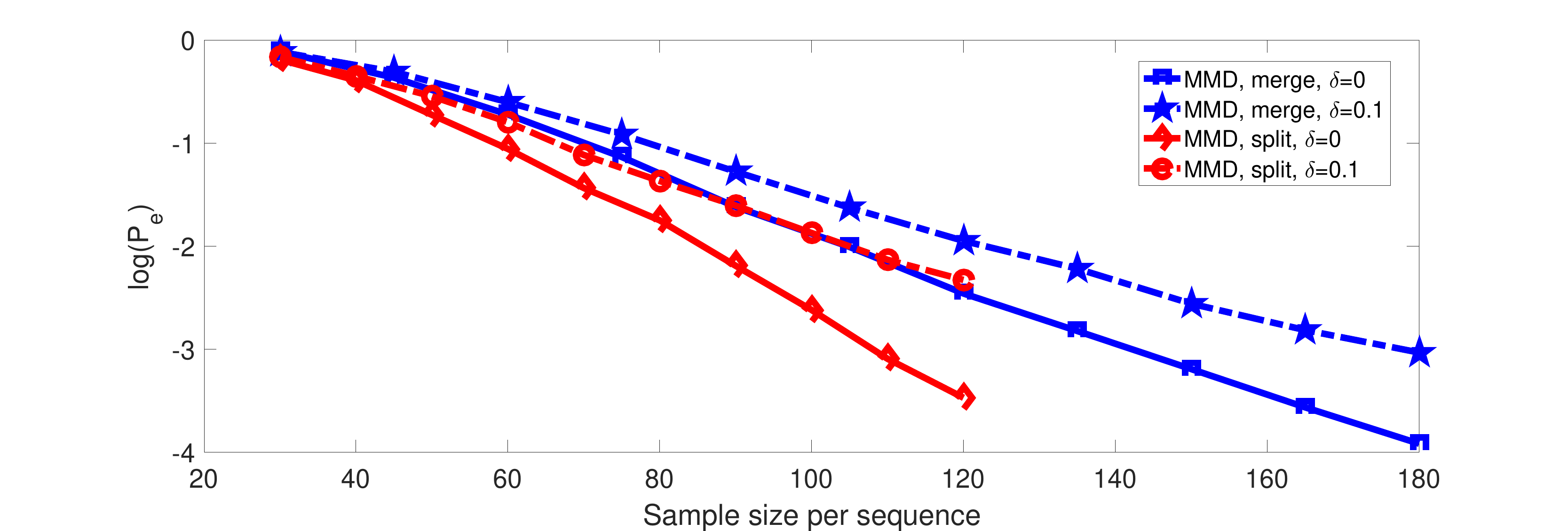}
		\label{fig:algorithm4_5_MMD-b}
	}
	\caption{Performance of Algorithms \ref{K-means-unknown-UC} and \ref{K-means-unknown-spl} for MMD}
	\label{fig:algorithm4_5_MMD}
\end{figure}
\subsection{Choice of $d_{th}$}
Note that in general $d_{th}=\omega d_{L} + \left(1-\omega\right)d_{H}$, where $\omega\in\left(0,1\right)$. Theorems \ref{theorem:KStest2} and \ref{theorem:KStest3} only establish the exponential consistency of Algorithms \ref{K-means-unknown-UC} and \ref{K-means-unknown-spl}, respectively. We now investigate the impact on performance given different $\omega$'s. One can observe from Fig.~\ref{fig:algorithm4_5_alpha} that the choice of $d_{th}$ has a significant impact on the performance of Algorithms \ref{K-means-unknown-UC} and \ref{K-means-unknown-spl}. The optimal $d_{th}$ depends on both the value of $\delta$ and the underlying distributions. Moreover, from Fig.~\ref{fig:algorithm4_5_alpha}\subref{fig:algorithm4_5_alpha_KS-a}-\subref{fig:algorithm4_5_alpha_KS-b}, we can see that a smaller $\omega$ which implies larger $d_{th}$ results in better performance for KS distance and the two algorithms always have similar performance. On the other hand, from Fig.~\ref{fig:algorithm4_5_alpha}\subref{fig:algorithm4_5_alpha_MMD-a}-\subref{fig:algorithm4_5_alpha_MMD-b}, we can see that $\omega=0.5$ is a good choice for MMD if $P_{e}$ is of interest and \eqref{eq:simu_kernel} is used as the kernel function. Moreover, when $\omega$ is small, Algorithm \ref{K-means-unknown-spl} outperforms Algorithm \ref{K-means-unknown-UC} under MMD.

\begin{figure}[t]
	\subfloat[KS distance ($n=100$, $\delta=0$)\label{subfig:4-1}]{%
		\includegraphics[width=0.5\textwidth, height=0.15\textheight]{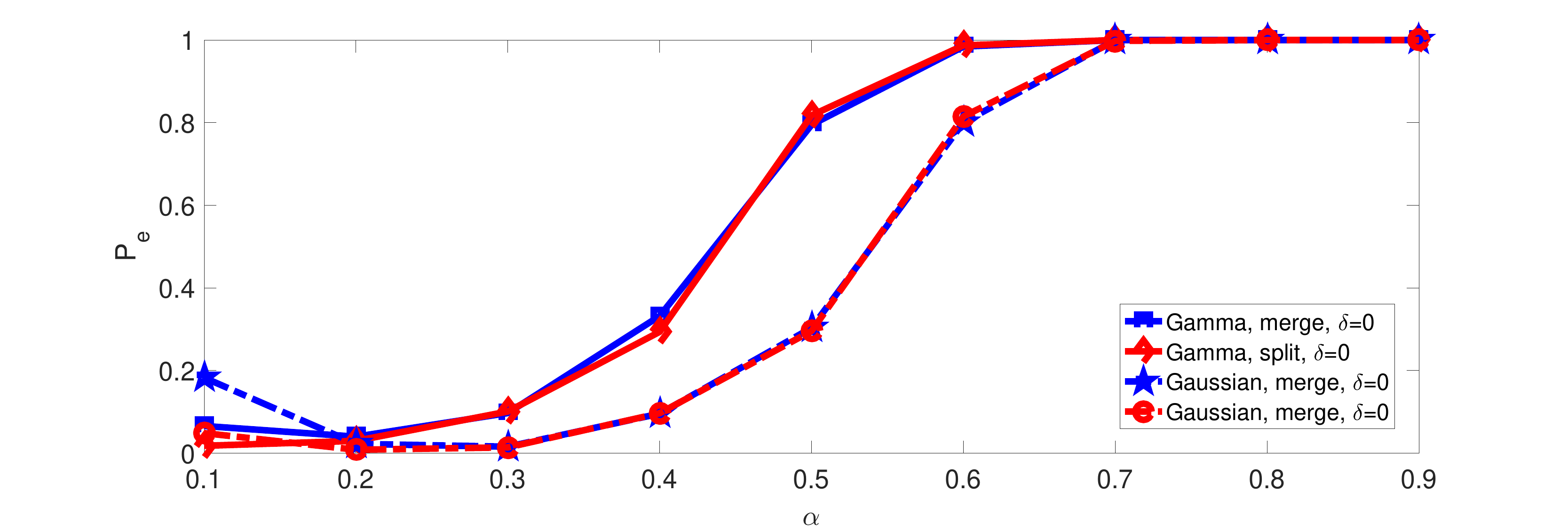}
		\label{fig:algorithm4_5_alpha_KS-a}
	}
	\hfill
	\subfloat[KS distance ($n=100$, $\delta=0.1$)\label{subfig:4-2}]{%
		\includegraphics[width=0.5\textwidth, height=0.15\textheight]{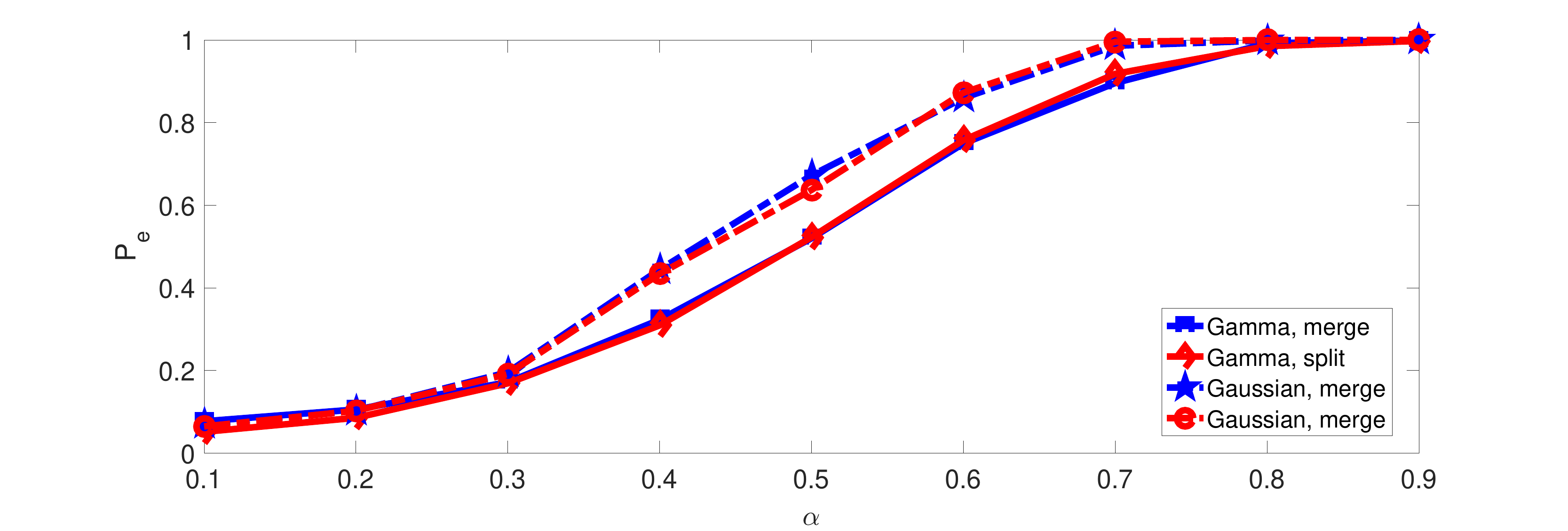}
		\label{fig:algorithm4_5_alpha_KS-b}
	}
	\hfill
	\subfloat[MMD ($n=80$, $\delta=0$)\label{subfig:4-3}]{%
		\includegraphics[width=0.5\textwidth, height=0.15\textheight]{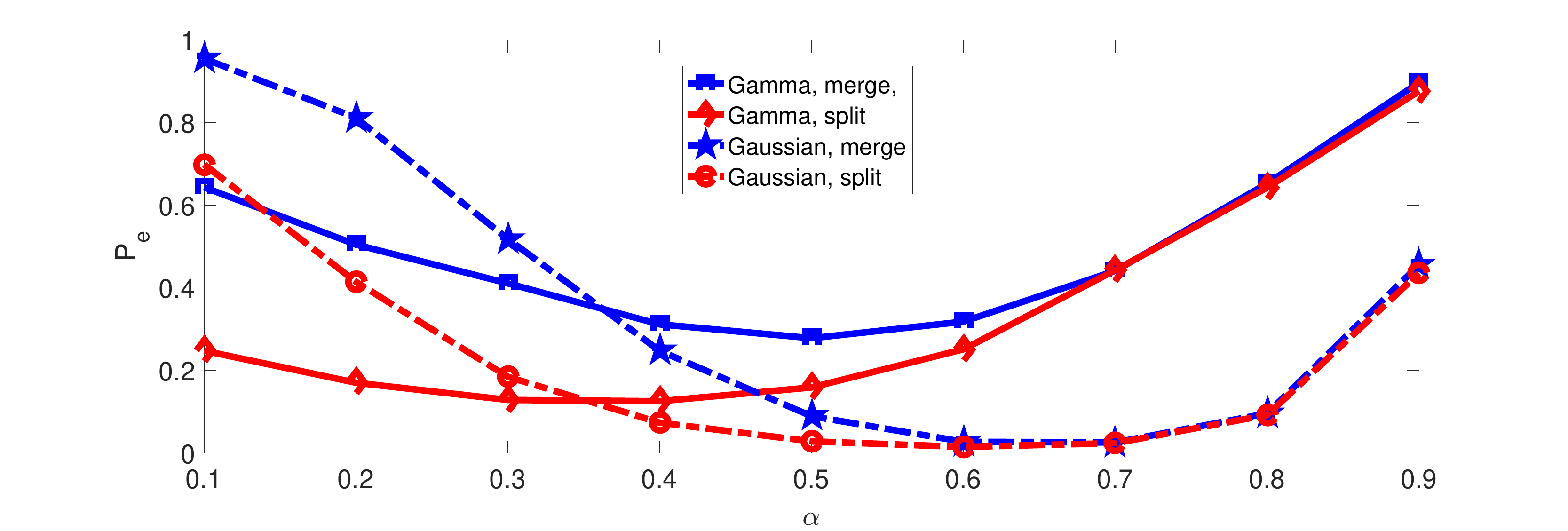}
		\label{fig:algorithm4_5_alpha_MMD-a}
	}
	\hfill
	\subfloat[MMD ($n=80$, $\delta=0.1$)\label{subfig:4-4}]{%
		\includegraphics[width=0.5\textwidth, height=0.15\textheight]{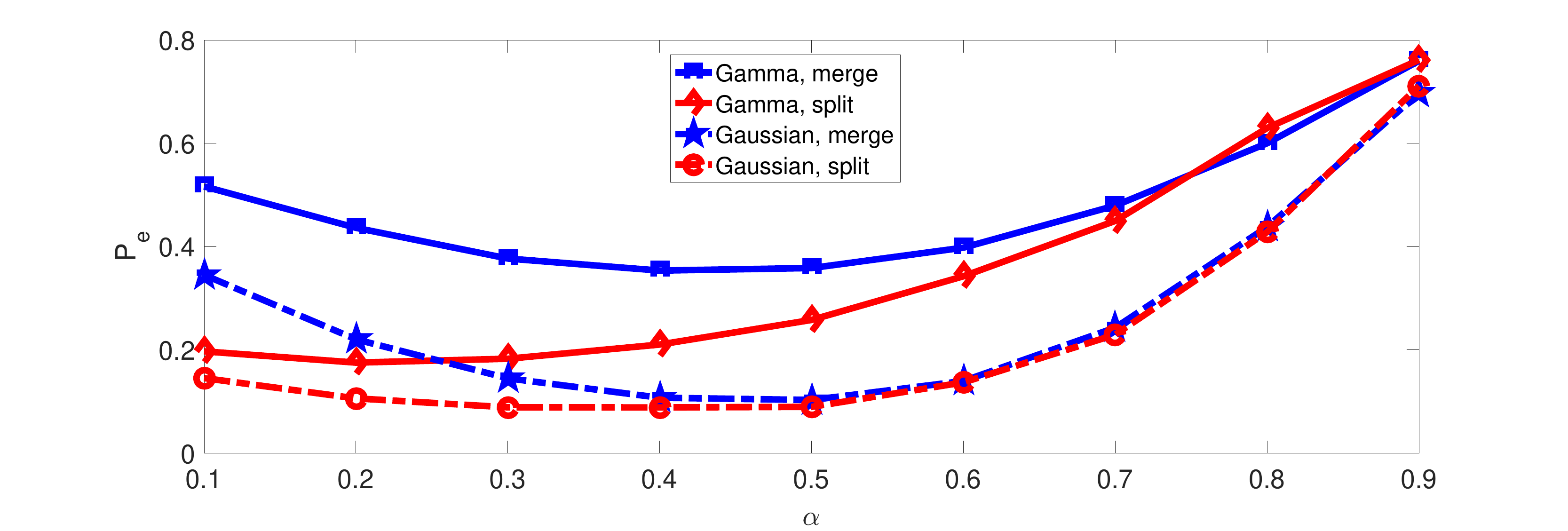}
		\label{fig:algorithm4_5_alpha_MMD-b}
	}
	\caption{Performance of Algorithms \ref{K-means-unknown-UC} and \ref{K-means-unknown-spl} given different $\omega$}
	\label{fig:algorithm4_5_alpha}
\end{figure}

\section{Conclusion}
This paper studied the k-medoids algorithm for clustering data sequences generated from composite distributions. The convergence of the proposed algorithms and the upper bound on the error probability were analyzed for both known and unknown number of clusters. The error probability of the proposed algorithms were characterized to decay exponentially for distance metrics that satisfied certain properties.
In particular, the KS distance and MMD were shown to satisfy the required condition, and hence the corresponding algorithms were exponentially consistent.\par
One possible generalization of the current work is to investigate the exponential consistency of other clustering algorithms given distributional distance metrics that satisfy the properties similar to that in Assumption \ref{assump}.



%



\appendix

\subsection{Technical Lemmas}\label{app:lemmas}
The following technical lemmas are used to prove Corollaries \ref{coro:knownK}, \ref{coro:unknownK-merge} and \ref{coro:unknownK-split}. All the data sequences in Lemmas \ref{lemma:KStest2} - \ref{lemma:MMD3} are assumed to consist of $n$ i.i.d. samples.
\begin{lemma}\emph{[Dvoretzky-Kiefer-Wolfowitz Inequality\cite{Massart1990}]}\label{lemma:KStest1}
	Suppose $\mathbf{x}$ consists of $n$ i.i.d. samples generated from $p$. Then
	\begin{equation*}
	P\big(d_{KS}\left(\mathbf{x},p\right)>\epsilon\big)\leq 2\exp\left(-2n\epsilon^{2}\right).
	\end{equation*}
\end{lemma}

\begin{lemma}\emph{[McDiarmid’s Inequality\cite{mcdiarmid1989}]}\label{lemma:mcd}
	Consider independent random variables $\mathbf{x}[1],\ldots,\mathbf{x}[n]\in \mathcal{X}$ and a mapping $f: \mathcal X^n \rightarrow \mathbb R$. If for all $i\in \{1,\ldots, n\}$, and for any $\tilde{x}\in \mathcal{X}$, there exist $c_i \le \infty$ for which
	\begin{equation*}
	\sup \limits_{\mathbf{x}\in \mathcal X^n , \tilde{x} \in \mathcal{X} } |f\left(\mathbf{x}\right)-f\left(\mathbf{x}^{\prime}\right)|\le c_i,
	\end{equation*}
	where 
	\begin{equation*}
	\mathbf{x}^{\prime}[j] = \begin{cases}
	\tilde{x} & \text{if } j=i,\\
	\mathbf{x}[j] & \text{otherwise},
	\end{cases}
	\end{equation*}
	then for all probability measure $p$ and every $\epsilon>0$,
	\begin{equation*}
	P\big(f\left(\mathbf{x}\right)-\mathbb E_\mathbf{x}[f\left(\mathbf{x}\right)] \geq \epsilon\big)< \exp\left(-\frac{2\epsilon^2}{\sum_{i=1}^{n}c_i^2}\right),	
	\end{equation*}
	where $\mathbb E_X[\cdot]$ denotes the expectation over the $n$ random variables $x_i \sim p$.
\end{lemma}

Lemmas \ref{lemma:KStest2} - \ref{lemma:MMD3} establish that the KS distance and MMD satisfy Assumption \ref{assump}. Moreover, the lemmas provided in \cite{Wang:ICASSP2018} are special cases of Lemmas \ref{lemma:KStest2}, \ref{lemma:KStest3} and \ref{lemma:KStest4} with $d_{L}=0$.
\begin{lemma}\label{lemma:KStest2}
	Suppose $\mathbf{x}_{j}\sim p_{j}$ for $j=1,2$, where $p_{j}\in \mathcal{P}$ and $d_{KS}\left(\mathcal{P}\right)=d_{L,ks}$. Then for any $d_{0}\in\left(d_{L,ks},\infty\right)$,
	\begin{equation*}
	P\big(d_{KS}\left(\mathbf{x}_{1},\mathbf{x}_{2}\right)>d_{0}\big) \leq 4\exp{\left(-\frac{n(d_{0}-d_{L,ks})^{2}}{2}\right)}.
	\end{equation*}
\end{lemma}
\begin{proof}
	Consider
	\begin{equation*}
	\begin{aligned}
	&\quad\: P\big(d_{KS}(\mathbf{x}_{1},\mathbf{x}_{2}) > d_{0}\big)\\
	&\leq P\big(d_{KS}(\mathbf{x}_{1},p_{1}) +d_{KS}(p_{1},p_{2}) + d_{KS}(\mathbf{x}_{2},p_{2})> d_{0}\big)\\
	& \leq P\big( d_{KS}(\mathbf{x}_{1},p_{1}) +d_{1} + d_{KS}(\mathbf{x}_{2},p_{2})> d_{0}\big) \\
	& \leq P\bigg(d_{KS}(\mathbf{x}_{1},p_{1}) > \frac{\hat{d}}{2} \bigg) + P\bigg(d_{KS}(\mathbf{x}_{2},p_{2}) > \frac{\hat{d}}{2} \bigg) \\
	& \leq 4 \exp{\big(-\frac{n\hat{d}^{2}}{2}\big)},
	\end{aligned}
	\end{equation*}
	where $d_{L,ks}<d_{1}<d_{0}$, $\hat{d}=d_{0}-d_{1}$ and $\lim_{d_{1}\downarrow d_{L,ks}} = d_{0} - d_{L,ks}$. The first inequality is due to the triangle inequality of the $L_{1}$-norm and the property of the supremum, and the last inequality is due to Lemma \ref{lemma:KStest1}. Therefore, by the continuity of the exponential function, we have
	\begin{equation*}
	P\big(d_{KS}(\mathbf{x}_{1},\mathbf{x}_{2})>d_{0}\big) \leq 4\exp{\bigg(-\frac{n(d_{0}-d_{L,ks})^{2}}{2}\bigg)}.\qedhere
	\end{equation*}
\end{proof}
Lemma \ref{lemma:KStest2} implies that the KS distance satisfies \eqref{eq:dist_condition1} for $d_{th}\in (d_{L,ks},\infty)$.

\begin{lemma}\label{lemma:MMD1}
	Suppose $\mathbf{x}_{j}\sim p_{j}$ for $j=1,2$, where $p_{j}\in \mathcal{P}$ and $\mathbb{MMD}^{2}(\mathcal{P})=d_{L,mmd}$. Then for any $d_{0}\in(d_{L,mmd},\infty)$,
	\begin{equation*}
	P\big(\mathbb{MMD}^{2}(\mathbf{x}_{1},\mathbf{x}_{2})>d_{0}\big)\leq \exp{\bigg(-\frac{n(d_{0}-d_{L,mmd})^{2}}{64\mathbb{K}^{2}}\bigg)}.
	\end{equation*}
\end{lemma}
\begin{proof}
	Without loss of generality, for any $\gamma\in\{1,\ldots,n\}$, consider $\mathbf{x}_{2}^{\prime}$, where
	\begin{equation*}
	\mathbf{x}_{2}^{\prime}[i] = \begin{cases}
	\tilde{x} & \text{if }i=\gamma,\\
	\mathbf{x}_{2}[i] & \text{otherwise}.
	\end{cases}
	\end{equation*}
	Note that given $\mathbf{x}_{1}$, $\mathbf{x}_{2}$ and $\mathbf{x}^{\prime}_{2}$,
	\begin{equation*}
	\begin{aligned}
	&\quad\:|\mathbb{MMD}^{2}(\mathbf{x}_{1},\mathbf{x}_{2}) - \mathbb{MMD}^{2}(\mathbf{x}_{1},\mathbf{x}_{2}^{\prime})\big|\\
	&=\bigg|\frac{2}{n(n-1)}\sum\limits_{i=1,i\ne \gamma}^n \bigg(g(\mathbf{x}_{2}[\gamma],\mathbf{x}_{1}[i])- g(\tilde{x},\mathbf{x}_{1}[i])\bigg)\\
	&\qquad-\frac{2}{n^2}\sum\limits_{i=1}^n \bigg(g(\mathbf{x}_{2}[\gamma],\mathbf{x}_{2}[i])-g(\tilde{x},\mathbf{x}_{2}[i])\bigg)\bigg|.
	\end{aligned}
	\end{equation*}
	Since $0\leq g(x,y)\leq \mathbb{K}$, by the triangle inequality of $L_{1}$-norm, we have
	\begin{equation*}
	\sup_{\mathbf{x}_{1}\in \mathcal{X}^{n},\mathbf{x}_{2}\in \mathcal{X}^{n},\tilde{x}\in \mathcal{X} }\big|\mathbb{MMD}^{2}(\mathbf{x}_{1},\mathbf{x}_{2}) - \mathbb{MMD}^{2}(\mathbf{x}_{1},\mathbf{x}_{2}^{\prime})\big|\leq \frac{8\mathbb{K}}{n},
	\end{equation*}
	where $\mathcal{X} = \mathbb{R}^{m}$ for any fixed $m$.
	Moreover, notice that $\mathbb{E}[\mathbb{MMD}^{2}(\mathbf{x}_{1},\mathbf{x}_{2})] = \mathbb{MMD}^{2}(p_{1},p_{2})\leq d_{L,mmd}$. Then
	\begin{equation*}
	\begin{aligned}
	&\quad\:P\bigg(\mathbb{MMD}^{2}(\mathbf{x}_{1},\mathbf{x}_{2})>d_{0}\bigg)\\
	&\leq P\big(\mathbb{MMD}^{2}(\mathbf{x}_{1},\mathbf{x}_{2}) - \mathbb{E}[\mathbb{MMD}^{2}(\mathbf{x}_{1},\mathbf{x}_{2})]> d_{0}-d_{L,mmd}\big)\\
	&\leq \exp{\bigg(-\frac{n(d_{0}-d_{L,mmd})^{2}}{64\mathbb{K}^{2}}\bigg)}.
	\end{aligned}
	\end{equation*}
	The last inequality is due to lemma \ref{lemma:mcd}.
\end{proof}
Lemma \ref{lemma:MMD1} implies that MMD satisfies \eqref{eq:dist_condition1} for $d_{th}\in (d_{L,mmd},\infty)$.

\begin{lemma}\label{lemma:KStest3}
	Suppose two distribution clusters $\mathcal{P}_{1}$ and $\mathcal{P}_{2}$ satisfy \eqref{eq:KSassumptionHT} under the KS distance. Assume that for $j=1,2$, $\mathbf{x}_{j}\sim p_{j}$ where $p_{j}\in \mathcal{P}_{j}$. Then for any $d_{0}\in (0,d_{H,ks})$,
	\begin{equation*}
	P\big(d_{KS}(\mathbf{x}_{1},\mathbf{x}_{2}) \leq d_{0}\big)\leq 4\exp{\bigg(-\frac{n(d_{H,ks}-d_{0})^{2}}{2}\bigg)}.
	\end{equation*}
\end{lemma}
\begin{proof}
	Similar to the proof of \cref{lemma:KStest2}, we have
	\begin{equation*}
	\begin{aligned}
	&\quad\:P\big(d_{KS}(\mathbf{x}_{1},\mathbf{x}_{2})\leq d_{0}\big)\\
	& \leq P\big( - d_{KS}(\mathbf{x}_{1},p_{1}) + d_{KS}(p_{1},p_{2}) - d_{KS}(\mathbf{x}_{2},p_{2}) < d_{0}\big) \\
	& \leq P\big( - d_{KS}(\mathbf{x}_{1},p_{1}) + d_{2} - d_{KS}(\mathbf{x}_{2},p_{2}) < d_{0}\big) \\
	& \leq P\bigg(d_{KS}(\mathbf{x}_{1},p_{1}) > \frac{\hat{d}}{2}\bigg) + P\bigg(d_{KS}(\mathbf{x}_{2},p_{2}) > \frac{\hat{d}}{2}\bigg)\\
	& \leq 4\exp{\big(-\frac{n\hat{d}^{2}}{2}\big)},
	\end{aligned}
	\end{equation*}
	where $d<d_{2}<d_{H}$, $\hat{d} = d_{2} - d_{0}$ and  $\lim_{d_{2}\uparrow d_{H,ks}} = d_{H,ks} - d_{0}$. The last inequality is due to Lemma \ref{lemma:KStest1}. Therefore, by the continuity of the exponential function, we have
	\begin{equation*}
	P\big(d_{KS}(\mathbf{x}_{1},\mathbf{x}_{2}) \leq d_{0}\big)\leq 4\exp{\bigg(-\frac{n(d_{H,ks}-d_{0})^{2}}{2}\bigg)}.\qedhere
	\end{equation*}
\end{proof}
Lemma \ref{lemma:KStest3} implies that the KS distance satisfies \eqref{eq:dist_condition2} for $d_{th}\in (0,d_{H,ks})$.

\begin{lemma}\label{lemma:MMD2}
	Suppose two distribution clusters $\mathcal{P}_{1}$ and $\mathcal{P}_{2}$ satisfy \eqref{eq:KSassumptionHT} under MMD. Assume that for $j=1,2$, $\mathbf{x}_{j}\sim p_{j}$ where $p_{j}\in \mathcal{P}_{j}$. Then for any $d_{0}\in (0,d_{H,mmd})$,
	\begin{equation*}
	P\big(\mathbb{MMD}^{2}(\mathbf{x}_{1},\mathbf{x}_{2})\leq d_{0}\big) \leq \exp{\bigg(-\frac{n(d_{H,mmd}-d_{0})^{2}}{64\mathbb{K}^{2}}\bigg)}.
	\end{equation*}
\end{lemma}
\begin{proof}
	Since $\mathbb{E}[\mathbb{MMD}^{2}(\mathbf{x}_{1},\mathbf{x}_{2})] = \mathbb{MMD}^{2}(p_{1},p_{2})\geq d_{H,mmd}$, similar to the proof of lemma \ref{lemma:MMD1}, we have
	\begin{equation*}
	\begin{aligned}
	&\quad\:P\big(\mathbb{MMD}^{2}(\mathbf{x}_{1},\mathbf{x}_{2})\leq d_{0}\big)\\
	&\leq P\big(\mathbb{E}[\mathbb{MMD}^{2}(\mathbf{x}_{1},\mathbf{x}_{2})] -\mathbb{MMD}^{2}(\mathbf{x}_{1},\mathbf{x}_{2})\geq d_{H,mmd}- d_{0}\big)\\
	&\leq \exp{\bigg(-\frac{n(d_{H,mmd}- d_{0})^{2}}{64\mathbb{K}^{2}}\bigg)}.
	\end{aligned}
	\end{equation*}
	The last inequality is due to lemma \ref{lemma:mcd}.
\end{proof}
Lemma \ref{lemma:MMD2} implies that MMD satisfies \eqref{eq:dist_condition1} for $d_{th}\in (0,d_{H,mmd})$.

\begin{lemma}\label{lemma:KStest4}\emph{\cite{Li2018}}
	Suppose two distribution clusters $\mathcal{P}_{1}$ and $\mathcal{P}_{2}$ satisfy \eqref{eq:KSassumptionHT} under the KS distance. Assume that for $j=1,2$, $\mathbf{x}_{j}\sim p_{j}$ where $p_{j}\in \mathcal{P}_{j}$. Then for any $\mathbf{x}_{3}\sim p_{3}$ where $p_{3}\in\mathcal{P}_{1}$,
	\begin{equation*}
	\begin{aligned}
	P\big(d_{KS}(\mathbf{x}_{1},\mathbf{x}_{3}) \geq d_{KS}(\mathbf{x}_{2},&\mathbf{x}_{3})\big)\leq 6\exp{\bigg(-\frac{n\Delta_{ks}^{2}}{8}\bigg)}.
	\end{aligned}
	\end{equation*}
\end{lemma}
Lemma \ref{lemma:KStest4} implies that the KS distance satisfies \eqref{eq:dist_condition3} for $d_{th}\in (d_{L,ks},d_{H,ks})$.

\begin{lemma}\label{lemma:MMD3}
	Suppose two distribution clusters $\mathcal{P}_{1}$ and $\mathcal{P}_{2}$ satisfy \eqref{eq:KSassumptionHT} under MMD. Assume that for $j=1,2$, $\mathbf{x}_{j}\sim p_{j}$ where $p_{j}\in \mathcal{P}_{j}$. Then for any $\mathbf{x}_{3}\sim p_{3}$ where $p_{3}\in\mathcal{P}_{1}$,
	\begin{equation*}
	\begin{aligned}
	P\big(\mathbb{MMD}^{2}(\mathbf{x}_{1},\mathbf{x}_{3}) \geq \mathbb{MMD}^{2}(\mathbf{x}_{2},\mathbf{x}_{3})\big)\leq \exp{\bigg(-\frac{n\Delta_{mmd}^{2}}{96\mathbb{K}^{2}}\bigg)}.
	\end{aligned}
	\end{equation*}
\end{lemma}
\begin{proof}
	See (35) in \cite{Li2018}.
\end{proof}
Lemma \ref{lemma:MMD3} implies that MMD satisfies \eqref{eq:dist_condition3} for $d_{th}\in (d_{L,mmd},d_{H,mmd})$.

\subsection{Proof of Main Results}
Define the following three events:
\begin{equation*}
\begin{aligned}
S_{1}(d_{th}) & = \big\{\exists k,k^{\prime}\in I_{1}^{K}, k\neq k^{\prime}, j\in I_{1}^{M_{k}}, j^{\prime}\in I_{1}^{M_{k^{\prime}}},\text{ s.t. }\\
&\qquad d(\mathbf{x}_{k,j},\mathbf{x}_{k^{\prime},j^{\prime}})\leq d_{th}\big\},\\
S_{2}(d_{th}) & = \big\{\exists k\in I_{1}^{K},\: j,j^{\prime}\in I_{1}^{M_{k}}\text{ s.t. }d(\mathbf{x}_{k,j},\mathbf{x}_{k,j^{\prime}})> d_{th} \big\},\\
S_{3} & = \big\{\exists k,k^{\prime}\in I_{1}^{K},\: k\neq k^{\prime},\: j_{1},j_{2}\in I_{1}^{M_{k}},\: j^{\prime}\in I_{1}^{M_{k^{\prime}}},\\
&\qquad \text{s.t. }d(\mathbf{x}_{k,j_{1}},\mathbf{x}_{k,j_{2}})\geq d(\mathbf{x}_{k,j_{1}},\mathbf{x}_{k^{\prime},j^{\prime}})\big\},
\end{aligned}
\end{equation*}
where $d_{th}\in (d_{L},d_{H})$.
Assume that the sequences $\mathbf{x}_{k,j}$'s and the corresponding distribution clusters $\mathcal{P}_{k}$'s satisfy Assumption \ref{assump}. By \eqref{eq:dist_condition2} - \eqref{eq:dist_condition3} and the union bound, we have
\begin{subequations}
	\begin{align}
	&P\big(S_{1}(d_{th})) \leq \sum_{k=1}^{K}\sum_{\substack{k^{\prime}=1\\k^{\prime}\neq k}}^{K}\sum_{j_{k}=1}^{M_{k}}\sum_{j_{k^{\prime}}=1}^{M_{k^{\prime}}}a_{1}e^{-bn} \leq M^{2}a_{1}e^{-bn},\label{eq:basis_ineq1}\\
	&P\big(S_{2}(d_{th})\big) \leq \sum_{k=1}^{K}\sum_{j_{k}=1}^{M_{k}}\sum_{j_{k^{\prime}}=1}^{M_{k^{\prime}}}a_{2}e^{-bn}\leq M^{2}a_{2}e^{-bn},\label{eq:basis_ineq2}\\
	&P\big(S_{3}\big) \leq \sum_{k=1}^{K}\sum_{j_{k}=1}^{M_{k}}\sum_{j_{k^{\prime}}=1}^{M_{k^{\prime}}}a_{3}e^{-bn}\leq M^{2}a_{3}e^{-bn}.\label{eq:basis_ineq3}
	\end{align}
\end{subequations}
The main idea of the proofs of Theorems \ref{theorem:KStest1}, \ref{theorem:KStest2} and \ref{theorem:KStest3} is to show that the error event at each iteration is a subset of $S_{1}(d_{th})\cup S_{2}(d_{th}) \cup S_{3}$.

\subsubsection{Proof of Theorem \ref{theorem:KStest1}}\label{proof:theorem:KStest1}
	The convergence of Algorithm \ref{K-means-known-C} results from the design of the algorithm. Consider the $(t-1)$-th clustering step and the $t$-th center update step. We have for $t\geq 1$,
	\begin{equation}\label{eq:pf-conv1-1}
	\begin{aligned}
	\sum_{k=1}^{K}\sum_{\mathbf{y}_{i}\in \mathcal{C}_{k}^{t-1,a}}d(\mathbf{y}_{i},\mathbf{c}_{k}^{t-1,a}) & \geq \sum_{k=1}^{K}\sum_{\mathbf{y}_{i}\in \mathcal{C}_{k}^{t-1,a}}d(\mathbf{y}_{i},\mathbf{c}_{k}^{t,a}).
	\end{aligned}
	\end{equation}
	Moreover, for the $t$-th center update and the $t$-th cluster update, we have for $t\geq 1$,
	\begin{equation}\label{eq:pf-conv1-2}
	\sum_{k=1}^{K}\sum_{\mathbf{y}_{i}\in
		\mathcal{C}_{k}^{t-1}}d(\mathbf{y}_{i},\mathbf{c}_{k}^{t,a}) \geq \sum_{l=1}^{K}\sum_{\mathbf{y}_{i}\in \mathcal{C}_{k}^{t}}d(\mathbf{y}_{i},\mathbf{c}_{k}^{t,a}).
	\end{equation}
	The equalities in \eqref{eq:pf-conv1-1} and \eqref{eq:pf-conv1-2} hold if and only if $\mathcal{C}_{k}^{t-1}=\mathcal{C}_{k}^{t}$ and $\mathbf{c}_{k}^{t-1,a} = \mathbf{c}_{k}^{t,a}$ for $k=1,\ldots,K$ respectively which implies the convergence of the algorithm. Since the sum of distances between the centers and sequences in the corresponding clusters is a non-negative value, then Algorithm \ref{K-means-known-C} converges after a finite number of iterations.\par
	Define for $t\geq 1$,
	\begin{equation*}
	\begin{aligned}
	E^{t} & = \{\text{After $t$-th iteration, there are $K_{1}$ centers generated} \\ 
	& \qquad\text{from $K_{2}$ distribution clusters} \}.
	\end{aligned}
	\end{equation*}
	where 
	\begin{equation*}
	K_{1}\begin{cases}
	>K_{2}& \text{ if } K_{2} = K,\\
	\geq K_{2} & \text{ if }K_{2}<K.
	\end{cases}
	\end{equation*}
	Similarly, define
	\begin{equation*}
	\begin{aligned}
	E^{0} & = \{\text{The center initialization obtains $K_{1}$ centers } \\ 
	& \qquad\text{generated from $K_{2}$ distribution clusters} \}.
	\end{aligned}
	\end{equation*}
	Then $E^{t}$ for $t\geq 0$ denotes the error event that centers are incorrectly chosen at the center initialization or the $t$-th center update. We first consider the error occurs at the initialization step. For Algorithm \ref{K-means-known-C},
	\begin{equation*}
	\begin{aligned}
	E^{0}& = \{\text{The center initialization results in $K$ centers gene-} \\ 
	& \qquad\text{rated from $K_2$ ($<K$) distribution clusters centers.}\}\\
	& = \{\exists k,l,l^{\prime}\in I_{1}^{K}, \:l\neq l^{\prime}\:\text{s.t. }\mathbf{c}_{l}^{0,a},\mathbf{c}_{l^{\prime}}^{0,a}\sim \mathcal{P}_{k}\}.
	\end{aligned}
	\end{equation*}
	Moreover, define
	\begin{equation*}
	\begin{aligned}
	E^{0}_{1} & =E^{0}\cap\big\{\exists l,l^{\prime}\in\{1,\ldots, K\} \text{ s.t. } d(\mathbf{c}_{l}^{0,a},\mathbf{c}_{l^{\prime}}^{0,a})\leq d_{th}\big\},\\
	E^{0}_{2} & = E^{0} \cap \big\{\exists l,l^{\prime}\in\{1,\ldots, K\} \text{ s.t. }d(\mathbf{c}_{l}^{0,a},\mathbf{c}_{l^{\prime}}^{0,a})> d_{th}\big\}.\\
	\end{aligned}
	\end{equation*}
	Then $E^{0} = E^{0}_{1}\cup E^{0}_{2}$. Without loss of generality, assume that $\mathbf{c}_{1}^{0,a},\ldots,\mathbf{c}_{K}^{0,a}$ are chosen sequentially at the center initialization step and $l<l^{\prime}$. Then $E^{0}_{1}$ implies that
	for all the sequences $\mathbf{z}\in\{\mathbf{y}_{i}\}_{i=1}^{M}\setminus \{\mathbf{c}^{0,a}_{m}\}_{m=1}^{l^{\prime}}$,
	\begin{equation*}
	\min_{m\in\{1,\ldots,l^{\prime}-1\}} d(\mathbf{c}_{m}^{0,a},\mathbf{z}) \leq d_{th}.
	\end{equation*}
	Thus, $E^{0}_{1}\subset S_{1}(d_{th})$. Then by \eqref{eq:basis_ineq1}, we have
	\begin{equation*}
	\begin{aligned}
	P\big(E^{0}_{1}\big)\leq P\big(S_{1}(d_{th})\big)\leq M^{2}a_{1}e^{-bn}.
	\end{aligned}
	\end{equation*}
	Moreover, since $E^{0}_{2}\subset S_{2}(d_{th})$, by \eqref{eq:basis_ineq2}, we have
	\begin{equation*}
	\begin{aligned}
	P\big(E^{0}_{2}\big)\leq P\big(S_{2}(d_{th})\big)\leq M^{2}a_{2}e^{-bn}.
	\end{aligned}
	\end{equation*}
	Thus, the error probability at the center initialization step is bounded as follows
	\begin{equation}\label{eq:alg2_initial}
	P\big(E^{0}\big) \leq   M^{2}(a_{1} + a_{2})e^{-bn}.
	\end{equation}
	\par
	We now consider the assignment step. Define for $t\geq 1$,
	\begin{equation*}
	\begin{aligned}
	H^{t} & = \{\text{The clustering result after the $t$-th cluster update}\\
	&\qquad \text{is incorrect}\},
	\end{aligned}
	\end{equation*}
	Moreover, define
	\begin{equation*}
	\begin{aligned}
	H^{0} & = \{\text{The clustering initialization is incorrect}\}.
	\end{aligned}
	\end{equation*}
	Since $E^{t}\subset H^{t-1}$ for $t\geq 1$, it is sufficient to obtain an upper bound on $P\big(H^{t}\big)$ which serves as the upper bound of $P(H^{t}\cup E^{t})$. Define
	\begin{equation*}
	\hat{H}^{t}_{1} = \begin{cases}
	H^{0}\setminus E^{0} & \text{ for }t=0,\\
	H^{t}\setminus \left(E^{0}\cup\left(\cup_{l=0}^{t-1}\left(H^{l}\right)\right)\right) & \text{ for }t\geq 1.
	\end{cases}
	\end{equation*}
	Then $E^{0}\cup\big(\cup_{t=1}^{T}H^{t}\big)=E^{0}\cup\big(\cup_{t=0}^{T}\hat{H}^{t}_{1}\big)$, which is the event that Algorithm \ref{K-means-known-C} makes an error before the first $T$ iterations complete. Moreover, $\hat{H}^{t}_{1}$ implies the event that an error occurs at the $t$-th cluster update step \emph{given} correct center update in the same iteration which is denoted by
	\begin{equation*}
	\begin{aligned}
	\bar{H}^{t}_{1} & =\big\{\exists k,k^{\prime},l,l^{\prime}\in I_{1}^{K},\; k\neq k^{\prime},\;j_{k}\in I_{1}^{M_{k}}\text{ s.t. }d(\mathbf{x}_{k,j_{k}},\\
	&\qquad\mathbf{c}_{l}^{t,a})\geq d(\mathbf{x}_{k,j_{k}},\mathbf{c}_{l^{\prime}}^{t,a}):\; \mathbf{c}_{l}^{t,a}\sim \mathcal{P}_{k},\: \mathbf{c}_{l^{\prime}}^{t,a}\sim\mathcal{P}_{k^{\prime}}\big\}.
	\end{aligned}
	\end{equation*}
	Then $P(\hat{H}^{t}_{1})\leq P\big(\bar{H}^{t}_{1}\big)$. Moreover, since $\bar{H}^{t}_{1}\subset S_{3}$, we have
	\begin{equation}\label{eq:alg2_iteration}
	\begin{aligned}
	P(\hat{H}^{t}_{1})\leq P(\bar{H}^{t}_{1}) \leq P\big(S_{3}\big) \leq M^{2}a_{3}e^{-bn}.
	\end{aligned}
	\end{equation}
	Therefore, by \eqref{eq:alg2_initial}, \eqref{eq:alg2_iteration} and the union bound, the error probability of Algorithm \ref{K-means-known-C} after $T$ iterations is bounded by
	\begin{equation}\label{eq:Pe_alg2}
	\begin{aligned}
	P_{e} & = P\big(E^{0}\cup \big(\cup_{t=0}^{T}\hat{H}^{t}_{1}\big)\big)\\
	&\leq M^{2}\big(a_{1} + a_{2} + (T+1)a_{3}\big)e^{-bn}.
	\end{aligned}
	\end{equation}
	\par

\subsubsection{Proof of Theorem \ref{theorem:KStest2}}\label{proof:theorem:KStest2}
	Note that the merge step may increase the sum of the KS distances between the centers and the sequences in the corresponding clusters. However, the deleting procedure in the merge step can only happen a finite number of times. Then by the same argument used in the previous proof, one can prove that Algorithm \ref{K-means-unknown-UC} converges after finite iterations.\par
	We first consider the initialization step. Define
	\begin{equation*}
	\begin{aligned}
	E_{3}^{0}&=E^{0}\cap \big\{K_{2}<K\big\},\\
	E_{4}^{0}& = E^{0}\cap \big\{K_{2}=K\big\}.
	\end{aligned}
	\end{equation*}
	Then $E^{0} = E_{3}^{0}\cup E_{4}^{0}$. Moreover, since
	\begin{equation*}
	\begin{aligned}
	E_{3}^{0} &\subset \big\{\exists k,k^{\prime}\in I_{1}^{K},\; j_{k}\in I_{1}^{M_{k}},\; j_{k^{\prime}}\in I_{1}^{M_{k^{\prime}}} \text{ s.t. }\\
	&\qquad d(\mathbf{x}_{k,j_{k}},\mathbf{x}_{k^{\prime},j_{k^{\prime}}})\leq d_{th}\big\},\\
	E_{4}^{0} &\subset \big\{\exists k\in I_{1}^{K},\; j_{k},j_{k}^{\prime}\in I_{1}^{M_{k}}, \text{ s.t. }d(\mathbf{x}_{k,j_{k}},\mathbf{x}_{k,j_{k}^{\prime}})> d_{th}\big\},
	\end{aligned}
	\end{equation*}
	then $E_{3}^{0} \subset S_{1}(d_{th})$ and  $E_{4}^{0} \subset S_{2}(d_{th})$. Thus, by \eqref{eq:basis_ineq1}, \eqref{eq:basis_ineq2}, we have
	\begin{equation*}
	\begin{aligned}
	P\big(E_{3}^{0}\big) &\leq P\big(S_{1}(d_{th})\big)\leq M^{2}a_{1}e^{-bn},\\
	P\big(E^{0}_{4}\big)& \leq P\big(S_{2}(d_{th})\big)\leq M^{2}a_{2}e^{-bn}.
	\end{aligned}
	\end{equation*}
	Therefore, by the union bound, the probability that an error occurs at the center initialization step is bounded by
	\begin{equation}\label{eq:alg4_initial}
	\begin{aligned}
	P\big(E^{0}\big) &\leq P\big(E^{0}_{3}\big) + P\big(E^{0}_{4}\big)\leq M^{2}a_{1}e^{-bn} +M^{2}a_{2}e^{-bn}.
	\end{aligned}
	\end{equation}
	\par
	We now consider the error that occurs during iterations. $E^{t}\subset H^{t-1}$ for $t\geq 1$ still holds. Furthermore, define an incorrect merge as the event that the distance between two centers generated from different distribution clusters is smaller than $d_{th}$. Let $D^{t}$ be the event that incorrect merges occur at the $t$-th ($t\geq 1 $) merge step. Thus we only need to bound $P\big(H^{t}\big)$ and $P\big( D^{t}\big)$. Let $B_{t_{1},t_{2}} =\big(\cup_{l=1}^{t_{1}}D^{l}\big)\cup\big(\cup_{l=0}^{t_{2}}H^{l}\big)$ for $t_{1}\geq 1$ and $t_{2}\geq 1$. Define
	\begin{equation*}
	\begin{aligned}
	\hat{D}^{t} &=
	\begin{cases}
	D^{1} &\; \text{for}\; t = 1\\
	D^{t}\setminus\big(E^{0}\cup B_{t-1,t-1}\big) & \;\text{for}\;t>1
	\end{cases},\\
	\end{aligned}
	\end{equation*}
	\begin{equation*}
	\begin{aligned}
	\hat{H}^{t}_{2} &=
	\begin{cases}
	H^{0} \setminus E^{0} &\; \text{for}\; t = 0\\
	H^{t}\setminus\big(E^{0}\cup B_{t,t-1}\big)&\; \text{for}\; t \geq 1
	\end{cases}.
	\end{aligned}
	\end{equation*}
	Then
	\begin{equation*}
	\begin{aligned}
	&\quad\:E^{0}\cup \big(\cup_{t=1}^{T}D^{t}\big)\cup\big(\cup_{t=0}^{T}H^{t}\big)\\
	&= E^{0}\cup\big(\cup_{t=1}^{T}\hat{D}^{t}\big)\cup\big(\cup_{t=0}^{T}\hat{H}^{t}_{1}\big),
	\end{aligned}
	\end{equation*}
	which denotes the event that an error occurs before $T$ iterations complete. Note that $\hat{D}^{t}$ implies the event that an error occurs at the $t$-th merge step \emph{given} correct center update in the same iteration, which is denoted by
	\begin{equation*}
	\begin{aligned}
	\bar{D}^{t} &= \big\{\exists k,k^{\prime}\in I_{1}^{K},\; k\neq k^{\prime},\; l\in I_{1}^{\hat{K}^{t-1}}, \text{ s.t. }d(\mathbf{c}_{l}^{t,a},\mathbf{c}_{l^{\prime}}^{t,a})\\
	&\qquad\leq  d_{th}:\;\mathbf{c}_{l}^{t,e}\sim \mathcal{P}_{k}, \mathbf{c}_{l^{\prime}}^{t,e}\sim \mathcal{P}_{k^{\prime}}\big\}.\\
	\end{aligned}
	\end{equation*}
	Then $P\big(\hat{D}^{t}\big)\leq P\big(\bar{D}^{t}\big)$ and $\bar{D}^{t}\subset S_{1}(d_{th})$. Thus, by \eqref{eq:basis_ineq1}, we have
	\begin{equation}\label{eq:alg4_iteration}
	\begin{aligned}
	P\big(\hat{D}^{t}\big)\leq P\big(\bar{D}^{t}\big)\leq P\big(S_{1}(d_{th})\big) \leq M^{2}a_{1}e^{-bn}.
	\end{aligned}
	\end{equation}
	Moreover, we have $P\big(\hat{H}_{2}^{t}\big)\leq P\big(\bar{H}_{2}^{t}\big)$, where
	\begin{equation*}
	\begin{aligned}
	\bar{H}^{t}_{2} & = \big\{\exists k,k^{\prime}\in I_{1}^{K},\; k\neq k^{\prime},\; j_{k}\in I_{1}^{M_{k}},\; l,l^{\prime}\in I_{1}^{\hat{K}^{t}},\text{ s.t.}\\ 
	&\quad d(\mathbf{x}_{k,j_{k}},\mathbf{c}_{l}^{t,e})\geq d(\mathbf{x}_{k,j_{k}},
	\mathbf{c}_{l^{\prime}}^{t,e}):\: \mathbf{c}_{l}^{t,e}\sim \mathcal{P}_{k},\:\mathbf{c}_{l^{\prime}}^{t,e}\sim \mathcal{P}_{k^{\prime}} \big\}.\\
	\end{aligned}
	\end{equation*}
	Note that $P(\bar{H}^{t}_{2})$ has the same upper bound as $P(\bar{H}^{t}_{1})$ in \eqref{eq:alg2_iteration}.
	Therefore, by \eqref{eq:alg4_initial}, \eqref{eq:alg2_iteration} and \eqref{eq:alg4_iteration}, the error probability after  $T$ iterations is bounded by
	\begin{equation}\label{eq:Pe_alg4}
	\begin{aligned}
	P_{e}&= P\big(Y^{0}\cup\big(\cup_{t=0}^{T}\hat{H}^{t}_{2}\big)\cup\big(\cup_{t=1}^{T}\hat{D}^{t}\big)\big)\\
	&\leq M^{2}\big((T+1)a_{1} + a_{2} + (T+1)a_{3}\big)e^{-bn}.
	\end{aligned}
	\end{equation}

\subsubsection{Proof of Theorem \ref{theorem:KStest3}}\label{proof:theorem:KStest3}
	Note that in the extreme case, splitting results in  each cluster containing only one sequence. Therefore, the maximum times of splitting is finite. Furthermore, if $\hat{K}$ does not change from the $(t-1)$-th to the $t$-th iteration, then $\mathcal{C}_{k}^{t-1}=\mathcal{C}_{k}^{t}$ for $k=1,\ldots,\hat{K}$. Otherwise,
	\begin{equation}
	\sum_{k=1}^{\hat{K}^{t-1}}\sum_{\mathbf{y}_{i}\in \mathcal{C}_{k}^{t-1,s}}d(\mathbf{c}_{k}^{t-1,s},\mathbf{y}_{i}) > \sum_{k=1}^{\hat{K}^{t}}\sum_{\mathbf{y}_{i}\in \mathcal{C}_{k}^{t,s}}d(\mathbf{c}_{k}^{t,s},\mathbf{y}_{i}),
	\end{equation}
	where $\mathbf{c}_{l}^{t,s}$ denotes the $l$-th center obtained at the $t$-th split step. Therefore, Algorithm \ref{K-means-unknown-spl} converges after finite iterations.\par
	Let $A^{t}$ be the event that the error occurs at the $t$-th split step. Then $A^{t}=A_{1}^{t}\cup A_{2}^{t}$, where
	\begin{equation*}
	\begin{aligned}
	A_{1}^{t} &= \big\{\text{The algorithm fails to split any cluster containing seq-} \\
	&\quad\: \text{uences generated by diffierent distribution clusters at}\\
	&\quad\: \text{the $t$-th iteration}\big\},\\
	A_{2}^{t} &= \big\{\text{The algorithm splits a cluster containing sequences} \\
	&\quad\: \text{generated by one distribution clusters at the $t$-th itera-}\\
	&\quad\: \text{tion}\big\}.\\
	\end{aligned}
	\end{equation*}
	Let $V^{t}$ denote the event that the clustering result at the $t$-th cluster update is incorrect. Then $A^{t}\cup V^{t}$ denotes the event that an error occurs at the $t$-th iteration. Define $\hat{A}^{t} = \hat{A}^{t}_{1}\cup \hat{A}^{t}_{2}$, where
	\begin{equation*}
	\begin{aligned}
	\hat{A}_{i}^{t} = \begin{cases}
	A^{1} &\;\text{for}\; t=1,\\
	A^{t}_{i}\setminus \big((\cup_{l=1}^{t-1}A^{l}) \cup (\cup_{l=1}^{t-1}V^{l})\big)&\;\text{for}\; t> 1,\\
	\end{cases}
	\end{aligned}
	\end{equation*}
	for $i=1,2$. Moreover, define
	\begin{equation*}
	\begin{aligned}
	\hat{V}^{t}= \begin{cases}
	V^{1}\setminus A^{1} &\;\text{for}\; t=1\\
	V^{t}\setminus \big((\cup_{l=1}^{t-1}V^{l})\cup (\cup_{l=1}^{t}A^{l})\big) &\;\text{for}\; t>1
	\end{cases}.
	\end{aligned}
	\end{equation*}
	Then $\big(\cup_{t=1}^{T}A^{t}\big)\cup\big(\cup_{t=1}^{T}V^{t}\big) = \big(\cup_{t=1}^{T}\hat{A}^{t}\big)\cup \big(\cup_{t=1}^{T}\hat{V}^{t}\big)$. Since $\hat{A}^{t}_{1}\subset S_{1}(d_{th})$ and $\hat{A}^{t}_{2}\subset S_{2}(d_{th})$, then we have for $t=1,\ldots,T$,
	\begin{equation*}
	\begin{aligned}
	P\big(\hat{A}^{t}_{1}\big)\leq P\big(S_{1}(d_{th})\big)\leq M^{2}a_{1}e^{-bn},\\
	P\big(\hat{A}^{t}_{2}\big)\leq P\big(S_{2}(d_{th})\big)\leq M^{2}a_{2}e^{-bn}.
	\end{aligned}
	\end{equation*}
	Moreover, since $P\big(\hat{A}^{t}\big) = P\big(\hat{A}_{1}^{t}\cup \hat{A}_{2}^{t}\big)$, by the union bound
	\begin{equation}\label{eq:alg5_spl1}
	P\big(\hat{A}^{t}\big) \leq M^{2}a_{1}e^{-bn} + M^{2}a_{2}e^{-bn}.
	\end{equation}
	Furthermore, by Definition \ref{def:correct_split}, $\hat{V}^{t}$ implies the following event
	\begin{equation*}
	\begin{aligned}
	\bar{V}^{t} &= \big\{\exists l,l^{\prime}\in I_{1}^{\hat{K}^{t}},\: k,k^{\prime}\in I_{1}^{K}\: k^{\prime}\neq k,\: j_{k}\in I_{1}^{M_{k}}\text{ s.t. }\\
	&\qquad d(\mathbf{x}_{k,j_{k}},\mathbf{c}_{l}^{t,s})\geq d(\mathbf{x}_{k,j_{k}},\mathbf{c}_{l^{\prime}}^{t,s}):\: \mathbf{c}_{l}^{t,s}\sim \mathcal{P}_{k},\\
	&\qquad\mathbf{c}_{l^{\prime}}^{t,s}\sim \mathcal{P}_{k^{\prime}} \big\}.
	\end{aligned}
	\end{equation*}
	Then, $P\big(\hat{V}^{t}\big)\leq P\big(\bar{V}^{t}\big)$ and $\bar{V}^{t}\subset S_{3}$. Thus, we have
	\begin{equation}\label{eq:alg5_spl2}
	P\big(\hat{V}^{t}\big)\leq P\big(\bar{V}^{t}\big)\leq M^{2}a_{3}e^{-bn}.
	\end{equation}
	Therefore, by \eqref{eq:alg5_spl1}, \eqref{eq:alg5_spl2} and the union bound, the error probability of Algorithm \ref{K-means-unknown-spl} after $T$ iterations is bounded by 
	\begin{equation}\label{eq:Pe_alg5}
	\begin{aligned}
	P_{e}& = P\big(\big(\cup_{t=1}^{T}\hat{A}^{t}\big)\cup\big(\cup_{t=1}^{T}\hat{V}^{t}\big)\big)\\
	&\leq M^{2}T\big(a_{1} + a_{2} + a_{3}\big)e^{-bn}.
	\end{aligned}
	\end{equation}



\ifCLASSOPTIONcaptionsoff
  \newpage
\fi



\bibliographystyle{IEEEtran}
\bibliography{reference}
\end{document}